\RequirePackage{fix-cm}
\documentclass[smallextended]{svjour3}       % onecolumn (second format)
\smartqed  % flush right qed marks, e.g. at end of proof
\usepackage{graphicx}
\newcommand{\ve}[2]{\langle #1 ,  #2 \rangle}

\newcommand{\RR}{\mathbf{R}}
\newcommand{\eqdef}{\stackrel{\text{def}}{=}}
\newcommand{\BL}{\begin{lem}} % added by Peter
\newcommand{\EL}{\end{lem}}   % added by Peter

\newcommand{\beq}{\begin{equation}}
\newcommand{\eeq}{\end{equation}}

\newcommand{\ba}{\begin{array}}
\newcommand{\ea}{\end{array}}
\newcommand{\beann}{\begin{eqnarray*}}
\newcommand{\eeann}{\end{eqnarray*}}
\newcommand{\bea}{\begin{eqnarray}}
\newcommand{\eea}{\end{eqnarray}}

\newtheorem{lem}{Proposition}
\newtheorem{thm}{Theorem}

\newcommand{\BEAS}{\begin{eqnarray*}}
\newcommand{\EEAS}{\end{eqnarray*}}
\newcommand{\BEA}{\begin{eqnarray}}
\newcommand{\EEA}{\end{eqnarray}}
\newcommand{\BEQ}{\begin{equation}}
\newcommand{\EEQ}{\end{equation}}
\newcommand{\BIT}{\begin{itemize}}
\newcommand{\EIT}{\end{itemize}}
\newcommand{\BNUM}{\begin{enumerate}}
\newcommand{\ENUM}{\end{enumerate}}
\newcommand{\BMI}{\begin{minipage}}
\newcommand{\EMI}{\end{minipage}}

\newcommand{\BA}{\begin{array}}
\newcommand{\EA}{\end{array}}
\newcommand{\BC}{\begin{center}}
\newcommand{\EC}{\end{center}}

\newcommand{\sgn}{\mathop{\bf sgn}}

\usepackage{amsmath,amssymb}
\usepackage{algorithm}
\usepackage{algorithmicx}
\usepackage{graphicx}
\usepackage{hyperref}
\usepackage{natbib}
\usepackage{xcolor}

\usepackage{xargs}                      % Use more than one optional parameter in a new commands
\usepackage[colorinlistoftodos,prependcaption,textsize=tiny]{todonotes}
\newcommandx{\unsure}[2][1=]{\todo[linecolor=red,backgroundcolor=red!25,bordercolor=red,#1]{#2}}
\newcommandx{\change}[2][1=]{\todo[linecolor=blue,backgroundcolor=blue!25,bordercolor=blue,#1]{#2}}
\newcommandx{\info}[2][1=]{\todo[linecolor=OliveGreen,backgroundcolor=OliveGreen!25,bordercolor=OliveGreen,#1]{#2}}
\newcommandx{\improvement}[2][1=]{\todo[linecolor=Plum,backgroundcolor=Plum!25,bordercolor=Plum,#1]{#2}}
\newcommandx{\thiswillnotshow}[2][1=]{\todo[disable,#1]{#2}}

\usepackage{array}
\usepackage{arydshln}
\authorrunning{Richt\'arik, Jahani, Ahipa{\c s}ao{\u g}lu, Tak\'a\v{c}}
\titlerunning{Alternating maximization: 8 sparse PCA formulations and efficient parallel codes}

\newcommand{\add}[1]{{#1}}

\newcommand{\remove}[1]{{\color{red}}}

\usepackage{ulem}

\usepackage{booktabs}
\usepackage{caption}
\usepackage{float}

\usepackage{array}
\usepackage{arydshln}
\begin{document}

\title{Alternating maximization: unifying framework for 8 sparse PCA formulations and efficient parallel codes\thanks{MT was partially supported by National Science Foundation grants CCF-1618717, CMMI-1663256 and  CCF-1740796. 
}
}
%\subtitle{Do you have a subtitle?\\ If so, write it here}

%\titlerunning{Short form of title}        % if too long for running head

\author{
Peter Richt\'{a}rik
  \and Majid Jahani   \and 
      Selin~Damla~Ahipa{\c s}ao{\u g}lu  
\and 
  Martin Tak\'{a}\v{c}
}

%\authorrunning{Short form of author list} % if too long for running head

\institute{Peter Richt\'arik \at
              Visual Computing Center, Al-Khawarizmi Building, Thuwal 23955, Saudi Arabia \\
              \email{peter.richtarik@kaust.edu.sa}           %  \\
%             \emph{Present address:} of F. Author  %  if needed
           \and
           Majid Jahani \at
           Industrial and Systems Engineering, 
200 West Packer Avenue, Bethlehem, PA 18015, USA
\email{majidjahani89@gmail.com}           
           \and
           Selin~Damla~Ahipa{\c s}ao{\u g}lu   \at
           Engineering Systems and Design, Sing. Univ. Tech. \& Design, 8 Somapah Road, Singapore
\email{ahipasa@gmail.com}
           \and
           Martin Tak\'a\v{c} \at
           Industrial and Systems Engineering, 
200 West Packer Avenue, Bethlehem, PA 18015, USA
\email{takac.mt@gmail.com}           
}

\date{Received: date / Accepted: date}
% The correct dates will be entered by the editor

\maketitle

\begin{abstract}
Given a multivariate data set, sparse principal component analysis (SPCA) aims to extract several linear combinations of the variables that together explain the variance in the data as much as possible, while controlling the number of nonzero loadings in these combinations. In this paper we consider 8 different optimization formulations for computing a single sparse loading vector: we employ two norms for measuring  variance (L2, L1) and two sparsity-inducing norms (L0, L1), which are used in two ways (constraint, penalty). Three of our formulations, notably the one with L0 constraint and L1 variance, have not been considered in the literature. We give a unifying reformulation which we propose to solve  via the alternating maximization (AM) method. We show that AM is equivalent to GPower for all formulations. Besides this, we provide 24 efficient parallel SPCA implementations: 3 codes (multi-core, GPU and cluster) for each of the 8 problems. Parallelism in the methods is aimed at i) speeding up computations (our GPU code can be 100 times faster than an efficient serial code written in C++), ii) obtaining solutions explaining more variance and iii) dealing with big data problems (our cluster code can solve a 357 GB problem in a minute).
% \PACS{PACS code1 \and PACS code2 \and more}
% \subclass{MSC code1 \and MSC code2 \and more}
\end{abstract}

\section{Introduction}\label{sec:Int}

Principal component analysis (PCA) is an indispensable tool used for dimension reduction in virtually all areas of science and engineering, from machine learning, statistics, genetics and finance to computer networks \citep{Jol86}. Let $A \in \RR^{n\times p}$ denote a data matrix encoding $n$ samples (observations) of $p$ variables (features). PCA aims to  extract a few linear combinations of the columns of $A$, called principal components (PCs), pointing in mutually orthogonal directions, together explaining as much variance in the data as possible. If the columns of $A$ are centered, the problem of extracting the first PC can be written as \begin{equation}\label{eq:classical_and_robust}\max\{ \| Ax \| : \|x\|_2 \leq 1\},\end{equation} where $\|\cdot\|$ is a suitable norm for measuring variance. The solution $x$ of this optimization problem is  called the loading vector, $Ax$ (normalized) is the first PC. Further PCs can be obtained in the same way with $A$ replaced by a new matrix in a process called deflation \citep{M08}. Classical PCA employs the $L_2$ norm in the objective; using the $L_1$ norm instead  may alleviate  problems caused by outliers in the data and hence leads to a robust PCA model \citep{Kwak08}. \add{Robust formulations using objective functions that are not functions of the covariance matrix (as in  \cite{CFF13}) are also possible, but these are beyond our investigation.} %Deflation methods for $L_1$ based SPCA are not as well studied as those for $L_2$ SPCA, but in principle similar techniques can be used.

%\textbf{Robustness.} Classical PCA employs the $L_2$-norm in the objective, which is known to be sensitive to outliers. When the data is contaminated, one may wish to measure the variance in $L_1$ norm instead, replacing $\|Ax\|_2$ by $\|Ax\|_1$ in the above formulation; this increases robustness \cite{MZX12}.

As normally there is no reason for the optimal loading vectors defining the PCs to be sparse, they are usually combinations of all of the variables. In some applications, however, sparse loading vectors enhance the \emph{interpretability} of the components and are easier to store, which leads to the idea to  \emph{induce} sparsity in the loading vectors. This problem and approaches to it are known collectively as sparse PCA (SPCA); for some \add{fundamental} work, refer to  \cite{ZHT04,MWS06,AEJL07,ABG08, shen08,LZ09,JNRS10,LT11,MZX12}. \add{Recent reviews on the subject can be found in \cite{Tren14} and \cite{HTW15}.} \add{In addition, recently, there has been great interest in establishing theoretical properties of sparse PCA including consistency, rates of
convergence, minimax risk bounds for estimating eigenvectors and principal subspaces
and detection under various and usually high-dimensional statistical models. See  \citep{AW09}, \citep{VL13}, \citep{VCLR13}, and \citep{LV2015}.} \add{The importance of robust and sparse models is getting more attention from various communities. For example, Robust Principal Component Analysis (RPCA), sometimes referred to as the Principal Component Pursuit (PCP), which decomposes a data matrix in a low-rank matrix and a sparse matrix has been investigated for video and signal processing \citep{CLMW11,HRSV16, AB16,BSJKZ17} and inducing sparsity into robust estimators has been successful in robust outlier detection \citep{HRSV16}.
} A popular way of incorporating a sparsity-inducing mechanism into \remove{the above} optimization formulation \add{\eqref{eq:classical_and_robust}} is via either a sparsity-inducing constraint or penalty. Two of the most popular functions for this are the $L_0$ and $L_1$ norm of the loading vector $x$ (the $L_0$ ``norm'' of $x$, denoted by $\|x\|_0$, is the number of nonzeros in $x$).

\subsection{Eight optimization formulations}

In this paper we consider 8 optimization formulations for extracting a single sparse loading vector (i.e., for computing the first PC) arising as combinations of the following three modeling factors: we use two norms for measuring variance (classical $L_2$ and robust $L_1$) and two sparsity-inducing (SI) norms (cardinality $L_0$ and $L_1$), which are used in two different ways (as a constraint or a penalty). All have the form
\begin{equation}OPT = \max_{x\in X} f(x),\label{eq:orig_formulation}\end{equation}
with $X\subset \RR^p$ and $f$ detailed in Table~1. Note that if we set $s=p$ in the constrained or $\gamma=0$ in the penalized versions, the sparsity-inducing functions stop having any effect\footnote{In the $L_1$ penalized formulations this can be seen from the inequality $\|x\|_1 \leq \sqrt{\|x\|_0}\|x\|_2$.} and we recover the classical and robust PCA \eqref{eq:classical_and_robust}. Choosing $1\leq s < p$, $\gamma>0$ will have the effect of directly enforcing or indirectly encouraging sparsity in the solution $x$.

\begin{table}[!ht]
\begin{center}
\tiny 
\begin{tabular}{c c c c l l}
   \toprule
  % after \\: \hline or \cline{col1-col2} \cline{col3-col4} ...
   \# & Variance & SI norm & SI norm usage &  \multicolumn{1}{c}{$X$} & \multicolumn{1}{c}{$f(x)$ \phantom{$\frac{\frac{1}{2}}{\frac{1}{2}}$}}  \\
\midrule
\textcolor{green!70!black}{  1 }& \textcolor{green!70!black}{ $L_2$} & \textcolor{green!70!black}{$L_0$ }& \textcolor{green!70!black}{ constraint} & \textcolor{green!70!black}{ $\{x \in \RR^p \;:\; \|x\|_2 \leq 1, \; \|x\|_0\leq s \}$       } & \textcolor{green!70!black}{ $\|Ax\|_2$} \phantom{$\frac{\frac{1}{2}}{\frac{1}{2}}$}   \\
      \hdashline
 \textcolor{green!70!black}{ 2 }& \textcolor{green!70!black}{$L_1$} & \textcolor{green!70!black}{$L_0$} & \textcolor{green!70!black}{constraint} & \textcolor{green!70!black}{$\{x \in \RR^p\;:\; \|x\|_2 \leq 1,\; \|x\|_0\leq s\}$    }    & \textcolor{green!70!black}{$\|Ax\|_1$}   \phantom{$\frac{\frac{1}{2}}{\frac{1}{2}}$}\\
      \hdashline
 \textcolor{blue}{  3} &\textcolor{blue}{ $L_2$ }&\textcolor{blue}{  $L_1$ }& \textcolor{blue}{ constraint } & \textcolor{blue}{ $\{x \in \RR^p\;:\; \|x\|_2 \leq 1,\; \|x\|_1\leq \sqrt{s}\}$ }& \textcolor{blue}{ $\|Ax\|_2$  } \phantom{$\frac{\frac{1}{2}}{\frac{1}{2}}$}\\
     \hdashline
  \textcolor{blue}{ 4 }& \textcolor{blue}{ $L_1$ }  & \textcolor{blue}{ $L_1$} &\textcolor{blue}{  constraint } &   \textcolor{blue}{$\{x \in \RR^p\;:\; \|x\|_2 \leq 1,\; \|x\|_1\leq \sqrt{s}\}$ } &  \textcolor{blue}{ $\|Ax\|_1$   } \phantom{$\frac{\frac{1}{2}}{\frac{1}{2}}$}\\
     \hdashline
   \textcolor{red}{5} & \textcolor{red}{$L_2$ }&\textcolor{red}{ $L_0$ }& \textcolor{red}{penalty  } & \textcolor{red}{$\{x\in \RR^p \;:\; \|x\|_2 \leq 1\}$                                   } & \textcolor{red}{$\|Ax\|_2^2 - \gamma \|x\|_0$} \phantom{$\frac{\frac{1}{2}}{\frac{1}{2}}$} \phantom{$\frac{\frac{1}{2}}{\frac{1}{2}}$}\\
    \hdashline
  \textcolor{red}{ 6} &  \textcolor{red}{$L_1$} & \textcolor{red}{$L_0$} & \textcolor{red}{penalty  }& \textcolor{red}{$\{x\in \RR^p\;:\; \|x\|_2 \leq 1\}$                                   } & \textcolor{red}{$\|Ax\|_1^2 - \gamma \|x\|_0$   \phantom{$\frac{\frac{1}{2}}{\frac{1}{2}}$}}\\
      \hdashline
 \textcolor{purple!50!blue}{ 7 }&\textcolor{purple!50!blue}{  $L_2$} & \textcolor{purple!50!blue}{ $L_1$} & \textcolor{purple!50!blue}{ penalty } & \textcolor{purple!50!blue}{ $\{x\in \RR^p\;:\; \|x\|_2 \leq 1\}$                        }            & \textcolor{purple!50!blue}{ $\|Ax\|_2 - \gamma \|x\|_1$ } \phantom{$\frac{\frac{1}{2}}{\frac{1}{2}}$} \\
     \hdashline
       \textcolor{purple!50!blue}{8} &  \textcolor{purple!50!blue}{$L_1$} & \textcolor{purple!50!blue}{$L_1$} & \textcolor{purple!50!blue}{penalty} & \textcolor{purple!50!blue}{$\{x\in \RR^p\;:\; \|x\|_2 \leq 1\}$  }                                  & \textcolor{purple!50!blue}{$\|Ax\|_1 - \gamma \|x\|_1$}   \phantom{$\frac{\frac{1}{2}}{\frac{1}{2}}$}\\
     \bottomrule
\end{tabular}
\end{center}
\caption{Eight sparse PCA optimization formulations; see \eqref{eq:orig_formulation}.}\label{tab:8main}
\end{table}

All 4 SPCA formulations of Table~1 involving $L_2$ variance were previously studied in the literature and are very popular. \remove{For instance,} \add{One of the earliest work, the well-known SCoTLASS (Simplified Component Technique-LASSO) method in \citet{Jolliffe03}, was for the $L_1$ penalized formulation. Although the original method is quite slow, faster numerical algorithms using projected gradient \citep{TJ06} and penalized matrix composition \citep{WTH09} was developed for SCoTLASS. The later one is an application of the conditional gradient algorithm as noted in \citep{LT11}.} \add{\citet{QLZ2013} considered a generalization of the problem with $L_1$ penalty, in which a mixed norm of $L_1$ and $L_2$ penalties is used.} \citet{AEJL07} solved a series of convex relaxations, based on semidefinite programming of the $L_0$ constrained $L_2$ variance problem, while \citet{ABG08} considered the $L_0$ penalized and constrained formulations. While, \citet{JNRS10} studied the $L_0$ and $L_1$ penalized versions, \citet{LT11} looked at all four. 
\add{Enforcing sparsity directly with an $L_0$ constrained formulation is NP-hard and it can't be approximated by an efficient approximation algorithm as shown in \cite{M-I16}. Therefore, there are only a few works that attempt to solve this  problem exactly; one recent notable study is \citet{BB2019}, which developed a branch and bound algorithm for this problem. In addition, \cite{BV16} discussed a hierarchy of optimality conditions for this problem.}

The $L_1$ constrained $L_1$ variance formulation was first proposed \remove{only recently,} by \citet{MZX12}. To the best of our knowledge, the remaining three $L_1$ variance formulations were not considered in the literature before. In particular, the $L_0$ constrained $L_1$ variance formulation is new---and is perhaps preferable as it directly constraints the cardinality of the loading vector $x$ without using any proxies.

%\note{The constrained formulations impose sparsity by restricting $x$ to a proper subset of $\RR^p$. In these, the target sparsity level, i.e., the maximum number of nonzero loadings in a component, is determined a priori. The feasible region is constrained to either the set of unit vectors with at most $s$ nonzero loadings (set of $s$-sparse unit vectors), where $s$ is the target sparsity level ($L_0$-constrained formulations) or a convex relaxation of this region ($L_1$-constrained formulations). On the other hand, the penalized formulations induce sparsity by introducing a penalty term into the objective function. The penalty term is a multiple, denoted by $\gamma$, of the cardinality of the solution ($L_0$-penalized formulations) or a relaxation of this ($L_1$-penalized formulations).
%}

\subsection{Reformulation and alternating maximization (AM) method} In all 8 formulations we introduce an additional (dummy) variable $y$, which allows us to propose a generic \emph{alternating maximization} method for solving them: i) for a fixed loading vector, find the best dummy variable (one maximizing the objective), then ii) fix the dummy variable and find the best loading vector; repeat steps i) and ii).  This and the resulting algorithms are described in detail in Section~\ref{sec:AMM}. The generic AM method is not limited to our choice of SPCA formulations. Indeed, it is applicable, for instance, if instead of measuring the variance using either the $L_1$ or the $L_2$ norm, we use any other norm. One critical feature shared by  the formulations in Table~1 is that steps i) and ii) of the AM method can be performed efficiently, in closed form, with the main computational burden in each step being a matrix-vector multiplication ($Ax$ in step i) and $A^Ty$ in step ii)).  Our method produces a sequence of loading vectors $x^{(k)}, \; k\geq 0$, with monotonically increasing values $f(x^{(k)})$.

Our approach of introducing a dummy variable and using AM is similar to that of \citet{JNRS10}, where it is done \emph{implicitly}, but mainly to that of \citet{Rich11SPARS}, where it is fully \emph{explicit}, albeit used for different purposes.

Besides providing a conceptual unification for solving all 8 formulations using a single algorithm (AM), the main theoretical result of this paper is establishing that, perhaps surprisingly, in all 8 cases, \emph{the AM method is equivalent to the GPower method}  \citep{JNRS10} applied to a certain derived objective function, with iterates being either the loading vectors or the dummy variables, depending on the formulation. This result is stated and proved in Section~\ref{sec:theory}.

% in the context of SPCA and subsequently identified as the conditional gradient method in \cite{LT11},

%with the search variable being either the loading vector or the dummy variable, depending on the formulation.

%We propose to solve the problems via a simple and natural alternating maximization scheme.

\subsection{Parallelism}  Besides giving  a new unifying framework and a generic algorithm for solving a number of SPCA formulations, 5 of which were previously proposed in the literature and 3 not, our further contribution is in providing efficient strategies for parallelizing AM at two different levels: i) running AM in parallel from multiple starting points in order to obtain a solution explaining more variance and ii) speeding up the linear  algebra involved. This is described in detail in Section~\ref{sec:par}.

Moreover, we provide parallel open-source code\footnote{Open source code with efficient implementations of the algorithms developed in this paper is published here: \remove{\href{https://code.google.com/p/24am/}{https://code.google.com/p/24am/}}
\add{\href{https://github.com/optml/24am}{https://github.com/optml/24am}}.} implementing these parallelization strategies,
for each of our 8 formulations, on 3 computing architectures: i) \emph{multi-core machine}, ii) \emph{GPU-enabled computer}, and iii) \emph{computer cluster}. We also provide a serial code; however,  as nearly all modern computers are multi-core, the serial implementation only serves the purpose of a benchmark against which once can measure parallelization speedup. Hence, we provide a total of $8\times 3 = 24$ parallel sparse PCA codes based on AM. Numerical experiments with our multi-core, GPU and cluster codes are performed in Section~\ref{sec:arch}.

Parallelism in our codes serves several purposes:
\begin{enumerate}
\item \emph{Speeding up computations.}  As described above, the AM method computes a matrix-vector multiplication at every iteration; this can be parallelized. We find that our GPU implementations are faster than our multi-core implementations, which are, in turn,  considerably faster than the benchmark single-core codes.
\item \emph{Obtaining solutions explaining more variance.} In some applications, such as in the computation of RIP constants for compressed sensing \citep{BT10}, it is critical that a PC is computed with as high explained variance as possible.  The output of our 8 subroutines depends on the starting point used; it only finds \remove{local} \add{stationary} solutions. Running them repeatedly from different starting points and keeping the solution with the largest objective value results in a PC explaining more variance. There are several ways in which this can be done, we implement~4 (NAI = ``naive'', SFA = ``start-from-all'', BAT = ``batches'' and OTF = ``on-the-fly''); details are given in Section~\ref{sec:par}.  A naive (NAI) approach is to do this sequentially; a different possibility is to run the method from several or all starting points in parallel (BAT, SFA), possibly asynchronously (OTF).
    This way at each iteration we need to perform a matrix-matrix multiplication which, when computed in parallel, is performed significantly faster compared to doing the corresponding number of parallel matrix-vector multiplications, one after another.
\item \emph{Dealing with big data problems.} If speed matters, for problems of small enough size we recommend using a GPU, if available. Since GPUs have stricter memory limitations than multi-core workstations (a typical GPU has 6GB RAM, a multi-core machine could have 20GB RAM), one may need to use a high-memory multi-core workstation if the problem size exceeds the GPU limit. However, for large enough (=big data) problems, one will need to use a cluster. Our cluster codes partition $A$, store parts of it on different nodes, and do the computations in a distributed way.
\end{enumerate}

\textbf{Notation.} By $x$ and $y$ we denote column vectors in $\RR^p$ and $\RR^n$, respectively. The coordinates of a vector are denoted by subscripts (eg., $x_1, x_2,\dots$) while iterates are denoted by superscripts in brackets (eg., $x^{(0)}$, $x^{(1)}$, $\dots$). We reserve the letter $k$ for the iteration counter. By $\|x\|_0$ we refer to the cardinality (number of nonzero loadings) of vector $x$. The $L_1, L_2$ and $L_\infty$ norms are defined by $\|z\|_1 = \sum_i |z_i|$, $\|z\|_2=(\sum_i z_i^2)^{1/2}$ and $\|z\|_\infty=\max_i|z_i|$, respectively. For a scalar $t$, we let $[t]_+ = \max\{0,t\}$ and by $\sgn(t)$ we denote the sign of $t$.

\section{Alternating Maximization (AM) Method} \label{sec:AMM}

As outlined in the previous section, we will solve \eqref{eq:orig_formulation} by introducing a dummy variable $y$ into each of the 8 formulations and apply an AM method to the reformulation. First, notice that for any pair of conjugate norms $\|\cdot\|$ and $\|\cdot\|^*$, we have, by definition,
\begin{equation}\label{eq:dualnorm}\|z\| = \max_{\|y\|^* \leq 1} y^T z.\end{equation}
In particular,  $\|\cdot\|_2^* = \|\cdot\|_2$ and $\|\cdot\|_{1}^* = \|\cdot\|_\infty$.

Now, let $Y:=\{y\in \RR^n \;:\; \|y\|_2 \leq 1\}$ for the $L_2$ variance formulations and $Y:=\{y\in \RR^n \;:\; \|y\|_\infty \leq 1\}$  for the $L_1$ variance formulations. Further, let $F(x,y)$ be the function obtained from $f(x)$ after replacing $\|Ax\|$ with $ y^T Ax$ (resp.\ $\|Ax\|^2$ with $(y^T Ax)^2$). Then, in view of the above, \eqref{eq:orig_formulation} takes on the equivalent form
\begin{equation}\label{eq:main}OPT = \max_{x\in X} \max_{y \in Y} F(x,y).\end{equation}
That is, the 8 problems from Table~\ref{tab:8main} can be reformulated into the form \eqref{eq:main}; the  details can be found in Table~\ref{tab:8REFORM}.

\begin{table}[!ht]
\begin{center}
\footnotesize
\begin{tabular}{c l l l }
   \toprule 
  % after \\: \hline or \cline{col1-col2} \cline{col3-col4} ...
   \# &   $X$ & $Y$ & $F(x,y)$ \phantom{$\frac{\frac{1}{2}}{\frac{1}{2}}$}\\
     \midrule
      %& & & \\
\textcolor{green!70!black}{   1 }& \textcolor{green!70!black}{  $\{x \in \RR^p \;:\; \|x\|_2 \leq 1, \; \|x\|_0\leq s \}$ }       &\textcolor{green!70!black}{  $\{y\in \RR^n \;:\; \|y\|_2 \leq 1\}$  }     &\textcolor{green!70!black}{  $y^T Ax$} \phantom{$\frac{\frac{1}{2}}{\frac{1}{2}}$}\\
    \hdashline
\textcolor{green!70!black}{  2 }& \textcolor{green!70!black}{ $\{x \in \RR^p \;:\; \|x\|_2 \leq 1,\; \|x\|_0\leq s\}$    }      & \textcolor{green!70!black}{$\{y\in \RR^n \;:\; \|y\|_\infty \leq 1\}$  }     & \textcolor{green!70!black}{$y^T Ax$  } \phantom{$\frac{\frac{1}{2}}{\frac{1}{2}}$}\\
    \hdashline
 \textcolor{blue}{ 3 }& \textcolor{blue}{ $\{x \in \RR^p \;:\; \|x\|_2 \leq 1,\; \|x\|_1\leq \sqrt{s}\}$ }  & \textcolor{blue}{$\{y\in \RR^n \;:\; \|y\|_2 \leq 1\}$     }  & \textcolor{blue}{$y^T Ax$  } \phantom{$\frac{\frac{1}{2}}{\frac{1}{2}}$}\\
   \hdashline
  \textcolor{blue}{  4} &   \textcolor{blue}{ $\{x \in \RR^p \;:\; \|x\|_2 \leq 1,\; \|x\|_1\leq \sqrt{s}\}$  } &  \textcolor{blue}{ $\{y\in \RR^n \;:\; \|y\|_\infty \leq 1\}$   }    &  \textcolor{blue}{ $y^T Ax$  } \phantom{$\frac{\frac{1}{2}}{\frac{1}{2}}$}\\
   \hdashline
 \textcolor{red}{ 5 }&  \textcolor{red}{$\{x \in \RR^p \;:\; \|x\|_2 \leq 1\}$   }                         & \textcolor{red}{$\{y\in \RR^n \;:\; \|y\|_2 \leq 1\}$    }   & \textcolor{red}{$(y^T Ax)^2 - \gamma \|x\|_0$ }\phantom{$\frac{\frac{1}{2}}{\frac{1}{2}}$}\\
     \hdashline
 \textcolor{red}{  6 }& \textcolor{red}{  $\{x \in \RR^p \;:\; \|x\|_2 \leq 1\}$  }                          & \textcolor{red}{ $\{y\in \RR^n \;:\; \|y\|_\infty \leq 1\}$   }    & \textcolor{red}{ $(y^T Ax)^2 - \gamma \|x\|_0$  } \phantom{$\frac{\frac{1}{2}}{\frac{1}{2}}$}\\
    \hdashline
  \textcolor{purple!50!blue}{ 7} & \textcolor{purple!50!blue}{  $\{x \in \RR^p \;:\; \|x\|_2 \leq 1\}$                     }       & \textcolor{purple!50!blue}{ $\{y\in \RR^n \;:\; \|y\|_2 \leq 1\}$   }    &\textcolor{purple!50!blue}{  $y^T Ax - \gamma \|x\|_1$ }  \phantom{$\frac{\frac{1}{2}}{\frac{1}{2}}$}\\
     \hdashline
 \textcolor{purple!50!blue}{ 8} & \textcolor{purple!50!blue}{ $\{x \in \RR^p \;:\; \|x\|_2 \leq 1\}$   }                         & \textcolor{purple!50!blue}{$\{y\in \RR^n \;:\; \|y\|_\infty \leq 1\}$  }    & \textcolor{purple!50!blue}{$y^T Ax - \gamma \|x\|_1$  } \phantom{$\frac{\frac{1}{2}}{\frac{1}{2}}$}\\
  \bottomrule
\end{tabular}
\end{center}
\caption{Reformulations of the problems from Table~\ref{tab:8main}.}\label{tab:8REFORM}
\end{table}

We propose to solve \eqref{eq:main} via Algorithm~\ref{alg:AM}.

\begin{algorithm}[h!]
\caption{Alternating Maximization (AM) Method.}
\label{alg:AM}
\begin{algorithmic}
 \State Select initial point $x^{(0)}\in\RR^p$; $k \leftarrow 0$
 \State \textbf{Repeat}
 \State \quad $y^{(k)} \leftarrow y(x^{(k)}) := \arg \max_{y\in Y} F(x^{(k)},y)$
 \State \quad $x^{(k+1)} \leftarrow x(y^{(k)}) := \arg \max_{x\in X} F(x, y^{(k)})$
 \State \textbf{Until} a stopping criterion is satisfied
\end{algorithmic}
\end{algorithm}

%\begin{equation}\label{eq:AM}y^{(k)} \leftarrow y(x^{(k)}) := \arg \max_{y\in Y} F(x^{(k)},y),\qquad x^{(k+1)} \leftarrow x(y^{(k)}) := \arg \max_{x\in X} F(x, y^{(k)}).\end{equation}

\subsection{Solving the subproblems} \label{sec:subproblems}

All 8 problems of Table~\ref{tab:8REFORM} enjoy the property that both of the steps (subproblems) of Algorithm~\ref{alg:AM} can be computed in closed form. In particular, each of these $8\times 2$ subproblems is of one of the 6 forms listed in Table~\ref{tab:5res}.

\begin{table}[ht!]
\tiny
\begin{center}
\begin{tabular}{ c c c c c }
 \toprule
  % after \\: \hline or \cline{col1-col2} \cline{col3-col4} ...
 Subproblem \# & $\phi(z)$ & $Z$ & $z^*$ & $\phi(z^*)$ \phantom{$\frac{\frac{1}{2}}{\frac{1}{2}}$}\\
  \midrule

S1 & $a^Tz$ \; or \; $(a^Tz)^2$ &          $\|z\|_2\leq 1$          &     $\tfrac{a}{\|a\|_2}$  &   $\|a\|_2$ \add{or 
$\|a\|_2^2$}
\phantom{$\frac{\frac{1}{2}}{\frac{1}{2}}$
}\\

 \hdashline

S2& $a^Tz$ &          $\|z\|_\infty\leq 1$          &     $\sgn(a)$   &   $\|a\|_1$ \phantom{$\frac{\frac{1}{2}}{\frac{1}{2}}$}\\

  \hdashline

S3 & $a^Tz$ &          $\|z\|_2\leq 1, \; \|z\|_0\leq s$          &     $\tfrac{T_s(a)}{\|T_s(a)\|_2}$  &   $\|T_s(a)\|_2$  \phantom{$\frac{\frac{1}{2}}{\frac{1}{2}}$}\\

  \hdashline

S4 &  $a^Tz$ &          $\|z\|_2\leq 1, \; \|z\|_1\leq \sqrt{s}$          &     $\tfrac{V_{\lambda_{s}(a)}(a)}{\|V_{\lambda_{s}(a)}(a)\|_2}$  &   $\lambda_{s}(a)\sqrt{s}+\|V_{\lambda_{s}(a)}(a)\|_2$  \phantom{$\frac{\frac{1}{2}}{\frac{1}{2}}$}\\

 \hdashline

S5 &  $(a^Tz)^2-\gamma\|z\|_0$ &          $\|z\|_2\leq 1$          &     $\tfrac{U_\gamma(a)}{\|U_\gamma(a)\|_2}$  &   $\|U_\gamma(a)\|_2^2-\gamma\|U_\gamma(a)\|_0$ \phantom{$\frac{\frac{1}{2}}{\frac{1}{2}}$} \\

\hdashline

S6 & $a^Tz-\gamma\|z\|_1$ &          $\|z\|_2\leq 1$         &     $\tfrac{V_\gamma(a)}{\|V_\gamma(a)\|_2}$  &   $\|V_\gamma(a)\|_2 $   \phantom{$\frac{\frac{1}{2}}{\frac{1}{2}}$}\\

\bottomrule
\end{tabular}
\end{center}
\caption{Closed-form solutions of AM subproblems; $z^* := \arg\max_{z\in Z}{\phi(z)}$.}\label{tab:5res}
\end{table}

The proofs of these elementary results, many of which are of folklore nature, can be found, for instance, in \citep{LT11} (and partially in \citep{JNRS10}). The columns of Table~\ref{tab:5res}, from left to right, correspond to the objective function, feasible region, maximizer (optimal solution) and maximum (optimal objective value). The first result will be used both with $z=x$ and $z=y$, the second result with $z=y$ and the remaining four results with $z=x$.

Table~\ref{tab:5res} is brief at the cost of referring to a number of operators ($T_s,  U_\gamma$, $V_\gamma: \RR^m \mapsto \RR^m$ and $\lambda_s:\RR^m \mapsto \RR$), which we will now define. For a given vector $a\in\RR^m$ and integer $s\in \{0,1,\dots,m\}$, by $T_s(a)\in \RR^m$ we denote the vector obtained from  $a$ by retaining only the $s$ largest components of $a$ in absolute value, with the remaining ones replaced by zero. For instance, for $a=(1,-4,2,5,3)^T$ and $s=2$ we have $T_s(a) = (0,-4,0,5,0)^T$. For $\gamma\geq 0$, we define operators $U_\gamma$ and $V_\gamma$ element-wise for $i=1,\dots,m$ as follows:
\begin{equation}\label{eq:U_gamma}(U_\gamma(a))_i := a_i[\sgn(a_i^2-\gamma)]_+,\end{equation}
\begin{equation}\label{eq:V_gamma}(V_\gamma(a))_i := \sgn(a_i)(|a_i|-\gamma)_+.\end{equation}
 Furthermore, we let \[\lambda_s(a):=\arg\min_{\lambda\geq 0}{\lambda\sqrt{s}+\|V_\lambda(a)\|_2},\] which is the solution of the one-dimensional dual of the optimization problem in line 4 of Table~\ref{tab:5res}.

\subsection{The AM method for all 8 SPCA formulations} Combining Algorithm~\ref{alg:AM} with the subproblem solutions given in Table~\ref{tab:5res}, the AM method for all our 8 SPCA formulations can be written down concisely; see Algorithm~\ref{alg:main}.

\begin{algorithm}[h!]
\caption{AM method for solving the 8 SPCA formulations of Table~\ref{tab:8REFORM}.}
\label{alg:main}
\begin{algorithmic}
 \State Select initial point $x^{(0)}\in\RR^p$; $k \leftarrow 0$
 \State \textbf{Repeat}
 \State \quad $u = A x^{(k)}$
 \State \qquad \textbf{If} $L_1$ variance \textbf{then} $y^{(k)} \gets \sgn(u)$
 \State \qquad \textbf{If} $L_2$ variance \textbf{then} $y^{(k)} \gets u/\|u\|_2$
 \State \quad $v = A^T y^{(k)}$
 \State \qquad \textbf{If} {$L_0$ penalty} \textbf{then} $x^{(k+1)} \gets U_\gamma(v)/\|U_\gamma(v)\|_2$
 \State \qquad \textbf{If} {$L_1$ penalty} \textbf{then} $x^{(k+1)} \gets V_\gamma(v)/\|V_\gamma(v)\|_2$
 \State \qquad \textbf{If} {$L_0$ constraint} \textbf{then} $x^{(k+1)} \gets T_s(v)/\|T_s(v)\|_2$
 \State \qquad \textbf{If} {$L_1$ constraint} \textbf{then} $x^{(k+1)} \gets V_{\lambda_{s}(v)}(v)/\|V_{\lambda_{s}(v)}(v)\|_2$
 \State \quad $k \leftarrow k+1$
 \State \textbf{Until} a stopping criterion is satisfied
\end{algorithmic}
\end{algorithm}

Note that in the methods described in Algorithm~\ref{alg:main} it is \add{(in theory)} not necessary to normalize the vector $U_\gamma(v)$ (resp. $V_\gamma(v)$, $T_s(v)$, and  $V_{\lambda_{s}(a)}(v)$) when computing $x^{(k+1)}$ since clearly the iterate $y^{(k+1)}$, which depends on $x^{(k+1)}$, is invariant under positive scalings of $x^{(k+1)}$, \add{and $y_k$ is being either normalized, or is computed using $\sgn$ function}. 
We have to remember, however, to normalize the output.
\add{When the matrix $A$ is not well conditioned, it is still recommended to normalize vectors $U_\gamma(v)$, $V_\gamma(v)$, $T_s(v)$, and  $V_{\lambda_{s}(a)}(v)$) to eliminate the effect of limited floating point precision.}

The method is terminated when a maximum number of iterations $maxIt$ is reached or when \[\frac{F(x^{(k+1)},y^{(k)})}{F(x^{(k)},y^{(k-1)})} \leq 1 + tol,\] whichever happens sooner.

\section{Equivalence of AM and GPower} \label{sec:theory}
GPower (generalized power method)  \citep{JNRS10} is a  simple algorithm for maximizing a convex function $\Psi$ on a compact set $\Omega$, which works via a ``linearize and maximize'' strategy. If by $\Psi'(z^{(k)})$ we denote an arbitrary subgradient of $\Psi$ at $z^{(k)}$, then GPower performs the following iteration: \begin{equation}z^{(k+1)} = \arg\max_{z\in \Omega} \{\Psi(z^{(k)}) + \ve{\Psi'(z^{(k)})}{z-z^{(k)}}\} = \arg \max_{z \in \Omega} \ve{\Psi'(z^{(k)})}{ z}.\label{eq:GPower}\end{equation}

The following theorem, our main result, gives a nontrivial insight into the relationship of AM and GPower, when the former is applied to solving any of the 8 SPCA formulations considered, and GPower is applied to a derived problem, as described by the theorem.

\begin{thm}[AM = GPower]  The AM and GPower methods are equivalent in the following sense:
\begin{enumerate}
\item For the 4 constrained sparse PCA formulations of Table~\ref{tab:8main}, the $x$ iterates of the AM method applied to the corresponding reformulation of Table~\ref{tab:8REFORM} are identical to the iterates of the GPower method as applied to the problem of maximizing the convex function \[F_Y(x) \eqdef \max_{y\in Y} F(x,y)\] on $X$, started from \add{a feasible} $x^{(0)}$, \add{such that 
$\|A x^{(0)}\| \neq 0$.}
\item    For the 4 penalized sparse PCA formulations of Table~\ref{tab:8main}, the $y$ iterates of the AM method applied to the corresponding reformulation of Table~\ref{tab:8REFORM} are identical to the iterates of the GPower method as applied to the problem of maximizing the convex function \[F_X(y) \eqdef \max_{x \in X} F(x,y)\] on $Y$, started from \add{a feasible} $y^{(0)}$
\add{(we assume that $y^{(0)}$, $s$ or $\gamma$ 
are chosen such that $F_X(y_0)>0$).
}
\end{enumerate}
\end{thm}

\begin{proof}
Recall that we wish to solve the problem
\[OPT =\max_{x \in X} f(x) = \max_{x\in X} \underbrace{\max_{y\in Y} F(x,y)}_{F_Y(x)} =  \max_{y\in Y} \underbrace{\max_{x\in X} F(x,y)}_{F_X(y)}.\]
We will now prove the equivalence for all 8 choices of $(f,X,Y,F)$ given in Tables~\ref{tab:8main} and \ref{tab:8REFORM}. In the proofs we will also refer to the closed form solutions of the subproblem (S1)--(S6), as detailed in Table~\ref{tab:5res}.

Consider first the constrained formulations: $1, 2, 3$ and $4$. By induction assume that the $k$-th  $x$-iterate ($x^{(k)}$) of AM is identical to the $k$-th iterate of GPower (for $k=0$ this is enforced by the assumption that GPower is started from $x^{(0)}$). By considering all 4 formulations individually, we will show that $x^{(k+1)}$ produced by AM and GPower are also identical.

\begin{enumerate}

\item[] \emph{Formulation 1}: Here we have \[f(x)=\|Ax\|_2, \qquad F(x,y) = y^T Ax,\]
\[X=\{x\in \RR^p \;:\; \|x\|_2\leq 1,\; \|x\|_0 \leq s\}, \qquad  Y=\{y \in \RR^{n}\;:\; \|y\|_{2}\leq 1\}.\]

First, note that
\[F_Y(x) = \max_{y\in Y}F(x,y) \overset{(S1)}{=} \|Ax\|_2,\]
the gradient of which is given by
\begin{equation}\label{eq:0980909s}F'_Y(x) = \frac{A^T A x}{\|Ax\|_2}.\end{equation}

Given $x^{(k)}$, in the AM method we have
\begin{equation}\label{eq:jsjkjk}y^{(k)} = \arg \max_{y\in Y} F(x^{(k)},y) \overset{(S1)}{=} \frac{Ax^{(k)}}{\|Ax^{(k)}\|_2}.\end{equation}

One iteration of GPower  started from $x^{(k)}$ will thus produce the iterate
\begin{eqnarray*}x^{(k+1)} \overset{\eqref{eq:GPower}}{=} \arg \max_{x\in X} \ve{ F_Y'(x^{(k)})}{ x} & \overset{\eqref{eq:0980909s}}{=} & \arg \max_{x \in X} \left\langle \frac{A^T Ax^{(k)}}{\|Ax^{(k)}\|_2},  x\right \rangle \\
&\overset{\eqref{eq:jsjkjk}}{=} & \arg \max_{x \in X} \ve{A^T y^{(k)}}{ x} \\
&\overset{(S3)}{=} & \frac{T_s(A^T y^{(k)})}{\|T_s(A^T y^{(k)})\|_2}.
\end{eqnarray*}

Observe that this is precisely how $x^{(k+1)}$ is computed in the AM method.

\item[] \emph{Formulation 2}: Here we have \[f(x)=\|Ax\|_1, \qquad F(x,y) = y^T Ax,\]
\[X=\{x\in \RR^p \;:\; \|x\|_2\leq 1,\; \|x\|_0 \leq s\}, \qquad  Y=\{y \in \RR^{n}\;:\; \|y\|_{\infty}\leq 1\}.\]

First, note that
\[F_Y(x) = \max_{y\in Y}F(x,y) \overset{(S2)}{=} \|Ax\|_1,\]
the \add{subgradient} of which is given by
\begin{equation}\label{eq:kjhkjhk8}F'_Y(x) = A^T \sgn(Ax).\end{equation}

Given $x^{(k)}$, in the AM method we have
\begin{equation}\label{eq:897987}y^{(k)} = \arg \max_{y\in Y} F(x^{(k)},y) \overset{(S2)}{=} \sgn(Ax^{(k)}).\end{equation}

One iteration of GPower  started from $x^{(k)}$ will thus produce the iterate
\begin{eqnarray*}x^{(k+1)} \overset{\eqref{eq:GPower}}{=} \arg \max_{x\in X} \ve{ F_Y'(x^{(k)})}{ x} & \overset{\eqref{eq:kjhkjhk8}}{=} & \arg \max_{x \in X} \left\langle A^T \sgn(Ax^{(k)}),  x\right \rangle \\
&\overset{\eqref{eq:897987}}{=} & \arg \max_{x \in X} \ve{A^T y^{(k)}}{ x} \\
&\overset{(S3)}{=} & \frac{T_s(A^T y^{(k)})}{\|T_s(A^T y^{(k)})\|_2}.
\end{eqnarray*}

Observe that this is precisely how $x^{(k+1)}$ is computed in the AM method.

\item[] \emph{Formulation 3}: Here we have \[f(x)=\|Ax\|_2, \qquad F(x,y) = y^T Ax,\]
\[X=\{x\in \RR^p \;:\; \|x\|_2\leq 1,\; \|x\|_1 \leq \sqrt{s}\}, \qquad  Y=\{y \in \RR^{n}\;:\; \|y\|_{2}\leq 1\}.\]

First, note that
\[F_Y(x) = \max_{y\in Y}F(x,y) \overset{(S1)}{=} \|Ax\|_2,\]
the gradient of which is given by
\begin{equation}\label{eq:aaaoiuo}F'_Y(x) = \frac{A^T Ax}{\|Ax\|_2}.\end{equation}

Given $x^{(k)}$, in the AM method we have
\begin{equation}\label{eq:lrtrtr}y^{(k)} = \arg \max_{y\in Y} F(x^{(k)},y) \overset{(S1)}{=} \frac{Ax^{(k)}}{\|Ax^{(k)}\|_2}.\end{equation}

One iteration of GPower  started from $x^{(k)}$ will thus produce the iterate
\begin{eqnarray*}x^{(k+1)} \overset{\eqref{eq:GPower}}{=} \arg \max_{x\in X} \ve{ F_Y'(x^{(k)})}{ x} & \overset{\eqref{eq:aaaoiuo}}{=} & \arg \max_{x \in X} \left\langle \frac{A^T Ax^{(k)}}{\|Ax^{(k)}\|_2},  x\right \rangle \\
&\overset{\eqref{eq:lrtrtr}}{=} & \arg \max_{x \in X} \ve{A^T y^{(k)}}{ x} \\
&\overset{(S4)}{=} & \frac{V_{\lambda_s(A^Ty^{(k)})}(A^Ty^{(k)})}{\|V_{\lambda_s(A^Ty^{(k)})}(A^Ty^{(k)})\|_2}.
\end{eqnarray*}

Observe that this is precisely how $x^{(k+1)}$ is computed in the AM method.

\item[] \emph{Formulation 4}: Here we have \[f(x)=\|Ax\|_1, \qquad F(x,y) = y^T Ax,\]
\[X=\{x\in \RR^p \;:\; \|x\|_2\leq 1,\; \|x\|_1 \leq \sqrt{s}\}, \qquad  Y=\{y \in \RR^{n}\;:\; \|y\|_{\infty}\leq 1\}.\]

First, note that
\[F_Y(x) = \max_{y\in Y}F(x,y) \overset{(S1)}{=} \|Ax\|_1,\]
\add{
the subgradient of which is given by
\begin{equation}\label{eq:aaaoiuo2}F'_Y(x) = A^T \sgn(Ax).\end{equation}
}

\add{
Given $x^{(k)}$, in the AM method we have
\begin{equation}\label{eq:lrtrtr2}y^{(k)} = \arg \max_{y\in Y} F(x^{(k)},y) \overset{(S2)}{=} \sgn(Ax^{(k)}).\end{equation}
}

One iteration of GPower  started from $x^{(k)}$ will thus produce the iterate
\add{
\begin{eqnarray*}x^{(k+1)} \overset{\eqref{eq:GPower}}{=} \arg \max_{x\in X} \ve{ F_Y'(x^{(k)})}{ x} & \overset{\eqref{eq:aaaoiuo2}}{=} & \arg \max_{x \in X} \left\langle 
A^T \sgn(Ax),  x\right \rangle \\
&\overset{\eqref{eq:lrtrtr2}}{=} & \arg \max_{x \in X} \ve{A^T y^{(k)}}{ x} \\
&\overset{(S4)}{=} & \frac{V_{\lambda_s(A^Ty^{(k)})}(A^Ty^{(k)})}{\|V_{\lambda_s(A^Ty^{(k)})}(A^Ty^{(k)})\|_2}.
\end{eqnarray*}
}

Observe that this is precisely how $x^{(k+1)}$ is computed in the AM method.

\end{enumerate}

Consider now the penalized formulations: $5, 6, 7$ and $8$. By induction assume that the $k$-th  $y$-iterate ($y^{(k)}$) of AM is identical to the $k$-th iterate of GPower (for $k=0$ this is enforced by the assumption that GPower is started from $y^{(0)}$). By considering all 4 formulations individually, we will show that $y^{(k+1)}$ produced by AM and GPower are also identical. Let $A = [a_1,\dots,a_p]$, i.e., the $i$-th column of $A$ is $a_i$.

\begin{enumerate}

\item[] \emph{Formulation 5}: Here we have \[f(x)=\|Ax\|_2^2 -\gamma \|x\|_0, \qquad F(x,y) = (y^T Ax)^2 - \gamma \|x\|_0,\]
\[X=\{x\in \RR^p \;:\; \|x\|_2\leq 1\}, \qquad  Y=\{y \in \RR^{n}\;:\; \|y\|_{2}\leq 1\}.\]

First, note that
\[F_X(y) = \max_{x\in X}F(x,y) \overset{(S5)}{=} \|U_\gamma(A^T y)\|_2^2 - \gamma \|U_\gamma(A^T y)\|_0 = \sum_{i=1}^p [(a_i^T y)^2-\gamma]_+,\]
the subgradient of which is given by
\begin{equation}\label{eq:0980hs908xx}F'_X(y) = 2\sum_{i=1}^p [\sgn((a_i^Ty)-\gamma)]_+ (a_i^T y) a_i \overset{\eqref{eq:U_gamma}}{=} 2 A U_\gamma(A^T y).\end{equation}

Given $y^{(k)}$, in the AM method we have
\begin{equation}\label{eq:js89djxx}x^{(k+1)} = \arg \max_{x\in X} F(x,y^{(k)}) \overset{(S5)}{=} \frac{U_\gamma(A^T y^{(k)})}{\|U_\gamma(A^T y^{(k)})\|_2}.\end{equation}

One iteration of GPower  started from $y^{(k)}$ will thus produce the iterate
\begin{eqnarray*}y^{(k+1)} \overset{\eqref{eq:GPower}}{=} \arg \max_{y\in Y} \ve{ F_X'(y^{(k)})}{ y} & \overset{\eqref{eq:0980hs908xx}}{=} & \arg \max_{\|y\|_\infty \leq 1} \ve{2 A U_\gamma(A^T y)}{ y} \\
&\overset{\eqref{eq:js89djxx}}{=} & \arg \max_{\|y\|_2\leq 1} \ve{Ax^{(k+1)}}{ y} \\
&\overset{(S1)}{=} & \frac{A x^{(k+1)}}{\|A x^{(k+1)}\|_2} .
\end{eqnarray*}

Observe that this is precisely how $y^{(k+1)}$ is computed in the AM method.

\item[] \emph{Formulation 6}: Here we have \[f(x)=\|Ax\|_1^2 -\gamma \|x\|_0, \qquad F(x,y) = (y^T Ax)^2 - \gamma \|x\|_0,\]
\[X=\{x\in \RR^p \;:\; \|x\|_2\leq 1\}, \qquad  Y=\{y \in \RR^{n}\;:\; \|y\|_{\infty}\leq 1\}.\]

First, note that
\[F_X(y) = \max_{x\in X}F(x,y) \overset{(S5)}{=} \|U_\gamma(A^T y)\|_2^2 - \gamma \|U_\gamma(A^T y)\|_0 = \sum_{i=1}^p [(a_i^T y)^2-\gamma]_+,\]
the subgradient of which is given by
\begin{equation}\label{eq:0980hs908}F'_X(y) = 2\sum_{i=1}^p [\sgn((a_i^Ty)-\gamma)]_+ (a_i^T y) a_i \overset{\eqref{eq:U_gamma}}{=} 2 A U_\gamma(A^T y).\end{equation}

Given $y^{(k)}$, in the AM method we have
\begin{equation}\label{eq:js89dj}x^{(k+1)} = \arg \max_{x\in X} F(x,y^{(k)}) \overset{(S5)}{=} \frac{U_\gamma(A^T y^{(k)})}{\|U_\gamma(A^T y^{(k)})\|_2}.\end{equation}

One iteration of GPower  started from $y^{(k)}$ will thus produce the iterate
\begin{eqnarray*}y^{(k+1)} \overset{\eqref{eq:GPower}}{=} \arg \max_{y\in Y} \ve{ F_X'(y^{(k)})}{ y} & \overset{\eqref{eq:0980hs908}}{=} & \arg \max_{\|y\|_\infty \leq 1} \ve{2 A U_\gamma(A^T y)}{ y} \\
&\overset{\eqref{eq:js89dj}}{=} & \arg \max_{\|y\|_\infty \leq 1} \ve{Ax^{(k+1)}}{ y} \\
&\overset{(S2)}{=} & \sgn(A x^{(k+1)}).
\end{eqnarray*}

Observe that this is precisely how $y^{(k+1)}$ is computed in the AM method.

\item[] \emph{Formulation 7}: Here we have \[f(x)=\|Ax\|_2 -\gamma \|x\|_1, \qquad F(x,y) = y^T Ax - \gamma \|x\|_1,\]
\[X=\{x\in \RR^p \;:\; \|x\|_2\leq 1\}, \qquad  Y=\{y \in \RR^{n}\;:\; \|y\|_{2}\leq 1\}.\]

Note that the functions  $y \mapsto F(x,y)$ are linear and that, by definition, $F_X(y) = \max_{x\in X} F(x,y)$. Moreover, note that the gradient of $y \mapsto F(x,y)$ at $y$ is equal to $Ax$. Hence, if $x$ is any vector that maximizes $F(x,y^{(k)})$ over $X$, then $Ax$ is a subgradient of $F_X$ at $y^{(k)}$. Note that this is precisely how $x^{(k+1)}$ is defined in the AM method: $x^{(k+1)} = \arg \max_{x\in X} F(x,y^{(k)})$. Hence, $Ax^{(k+1)}$ is a subgradient of $F_X$ at $y^{(k)}$ and one iteration of GPower  started from $y^{(k)}$ will produce the iterate
\[y^{(k+1)} \overset{\eqref{eq:GPower}}{=} \arg \max_{y\in Y} \ve{ F_X'(y^{(k)})}{ y} = \arg \max_{\|y\|_2 \leq 1} \ve{A x^{(k+1)}}{ y} \overset{(S1)}{=} \frac{A x^{(k+1)}}{\|A x^{(k+1)}\|_2} .\]
Observe that this is precisely how $y^{(k+1)}$ is computed in the AM method.

\item[] \emph{Formulation 8}: Here we have \[f(x)=\|Ax\|_1 -\gamma \|x\|_1, \qquad F(x,y) = y^T Ax - \gamma \|x\|_1,\]
\[X=\{x\in \RR^p \;:\; \|x\|_2\leq 1\}, \qquad  Y=\{y \in \RR^{n}\;:\; \|y\|_{\infty}\leq 1\}.\]

Note that the functions  $y \mapsto F(x,y)$ are linear and that, by definition, $F_X(y) = \max_{x\in X} F(x,y)$. Moreover, note that the gradient of $y \mapsto F(x,y)$ at $y$ is equal to $Ax$. Hence, if $x$ is any vector that maximizes $F(x,y^{(k)})$ over $X$, then $Ax$ is a subgradient of $F_X$ at $y^{(k)}$. Note that this is precisely how $x^{(k+1)}$ is defined in the AM method: $x^{(k+1)} = \arg \max_{x\in X} F(x,y^{(k)})$. Hence, $Ax^{(k+1)}$ is a subgradient of $F_X$ at $y^{(k)}$ and one iteration of GPower  started from $y^{(k)}$ will produce the iterate
\[y^{(k+1)} \overset{\eqref{eq:GPower}}{=} \arg \max_{y\in Y} \ve{ F_X'(y^{(k)})}{ y} = \arg \max_{\|y\|_\infty \leq 1} \ve{A x^{(k+1)}}{ y} \overset{(S2)}{=} \sgn(A x^{(k+1)}) .\]
Observe that this is precisely how $y^{(k+1)}$ is computed in the AM method.
\end{enumerate}
\end{proof}

Having established equivalence between AM and GPower,  convergence to a \remove{local solution} \add{stationary point} of the AM method for all 8 SPCA formulations follows from the theory developed by \citet{JNRS10} and \citet{LT11}.

\section{Embedding AM within a Parallel Scheme}\label{sec:par}

In this section we describe several approaches for embedding Algorithm~\ref{alg:main} (AM) within a parallel scheme for solving $l$ identical SPCA problems, started from a number of starting points, $x^{(0,1)}, \dots, x^{(0,l)}$. This is done in order to obtain a loading vector explaining more variance and will be discussed in more detail in Section~\ref{sec:hunt}.

As we will see, it may not necessarily be most efficient to solve \emph{all} $l$ problems simultaneously. Instead, we  consider a class of parallelization schemes where we divide the $l$ problems  into ``batches'' of $r$ problems each, and solve each batch of $r$ problems simultaneously. In this setting at each iteration we need to perform identical operations in parallel, notably matrix-vector multiplications $Ax^{(k,1)}, \dots, Ax^{(k,r)}$ and $A^Ty^{(k,1)}, \dots, A^Ty^{(k,r)}$. It is useful to view the sequence of matrix-vector products as a single matrix-matrix product, e.g., $A[x^{(k,1)},\dots, x^{(k,r)}]$ in the first case, and use optimized libraries for parallelization. This simple trick leads to considerable speedups when compared to other approaches. We use similar ideas for the parallel evaluation of the operators. Note that even in the $l=1$ case, i.e, if we wish to run SPCA from a single starting point only, there is scope for parallelization of the matrix-vector products and function evaluations. Hence, parallelization in our method serves two purposes:
\begin{enumerate}
\item to obtain solutions explaining more variance by solving the problem from several starting points (we choose $l>1$),
\item to speed up computations by parallelizing the linear algebra involved (this applies to both $l=1$ and $l>1$ cases).
\end{enumerate}

\begin{figure}[!ht]
 \centering
\includegraphics[width=4in]{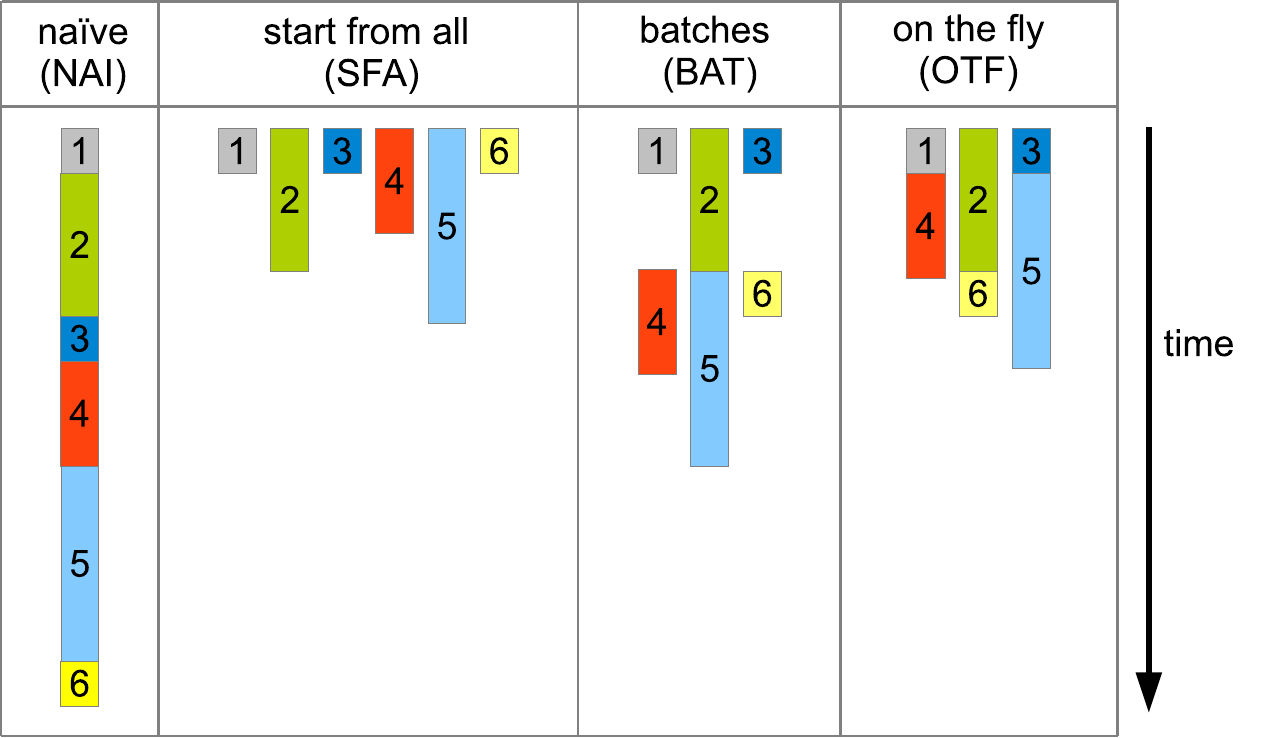}
 \caption{Four ways of embedding Algorithm 2 (AM) in a parallel scheme. In this example we run AM on the same problem $l=6$ times, using different (random) starting points.}
 \label{fig:4parallel}
\end{figure}

In particular, in this section we describe 4 parallelization approaches:
\begin{itemize}
\item NAI = ``naive'' ($r=1$),
\item SFA = ``start-from-all'' ($r=l$),
\item BAT = ``batches'' ($1\leq r \leq l$)
\item OTF = ``on-the-fly'' (BAT improved by a dynamic replacement strategy to reduce idle time).
\end{itemize}

The working of these 4 approaches is illustrated in Figure~\ref{fig:4parallel} in a situation with $l=6$. In what follows we describe the methods informally, in a narrative style, with a suitable choice of numerical experiments illustrating the differences between the ideas.

\subsection{The hunt for more explained variance} \label{sec:hunt}
% case L_2 Real Dataset
\begin{figure}[h!]
 \centering
 \includegraphics[scale = 0.2]{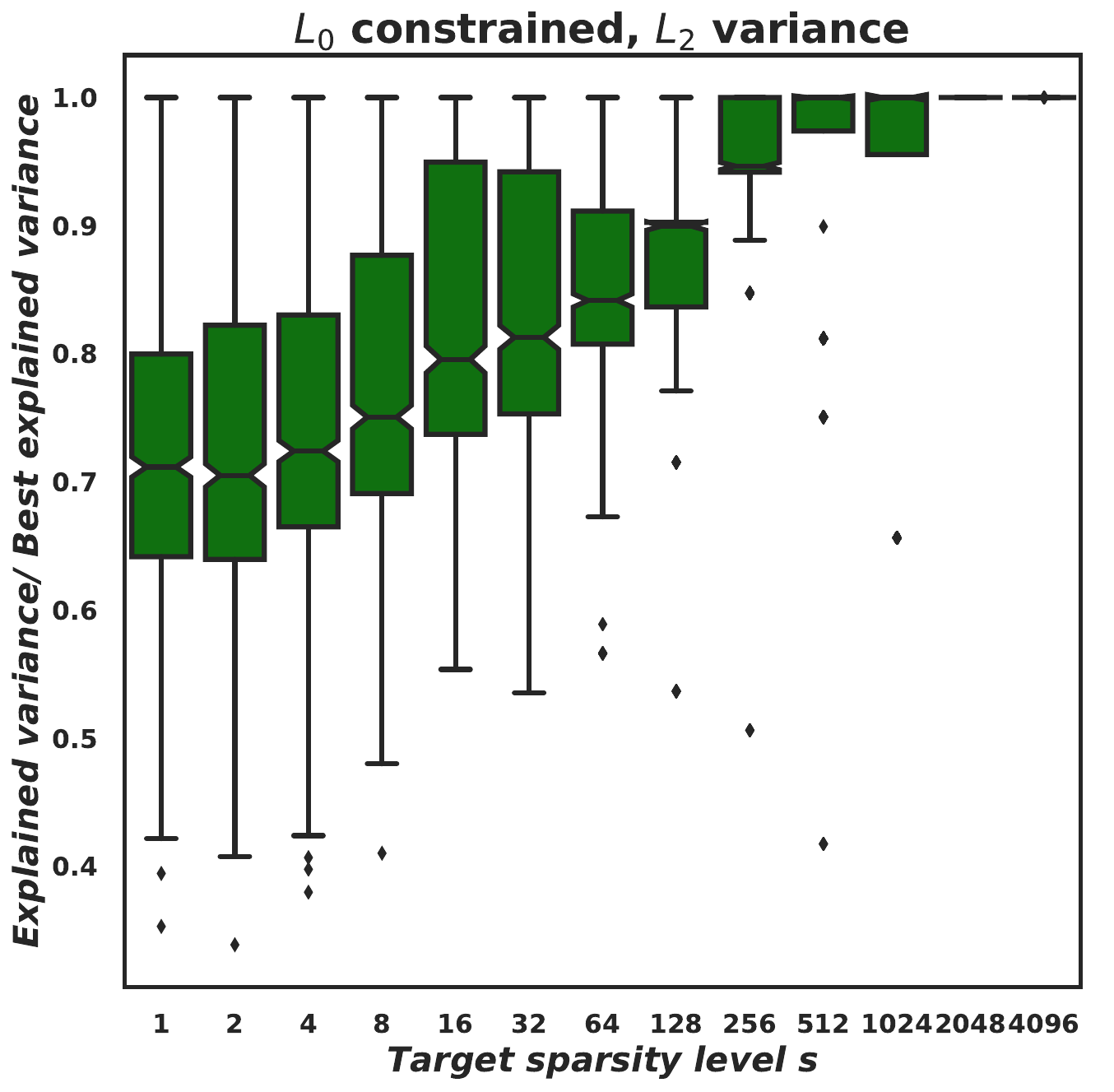}
 \includegraphics[scale = 0.2]{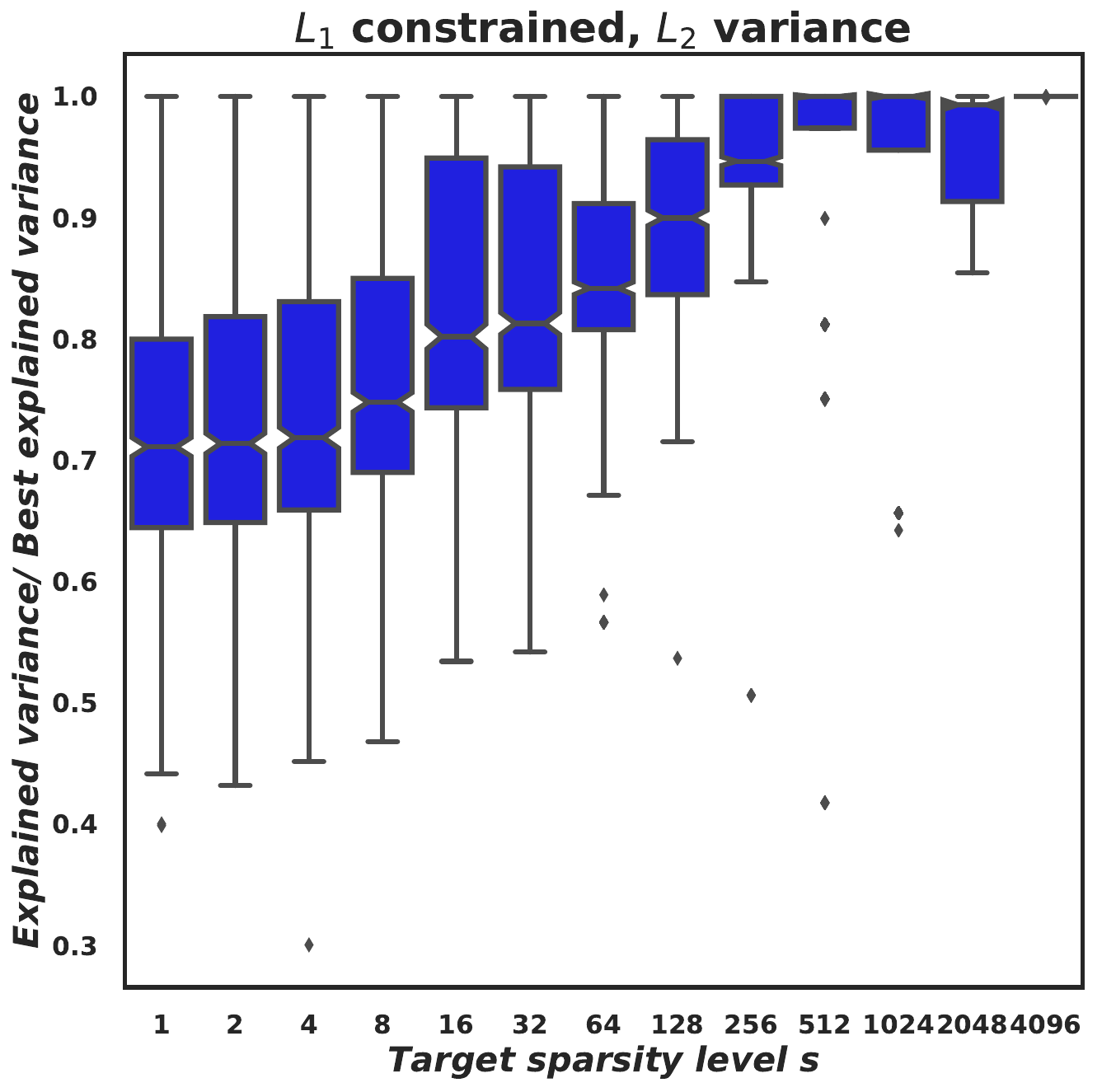}
  
 \includegraphics[scale = 0.2]{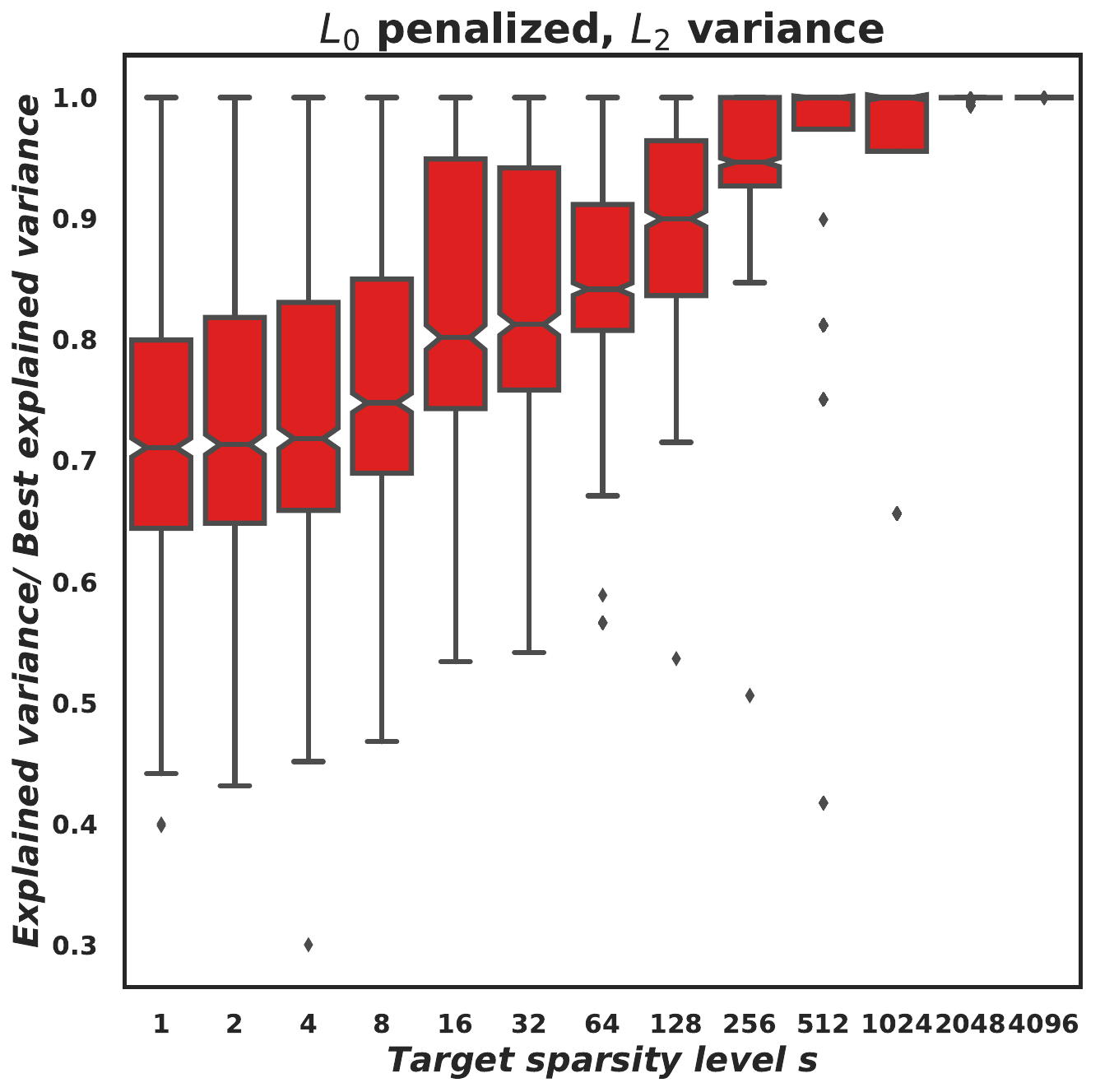}
 \includegraphics[scale = 0.2]{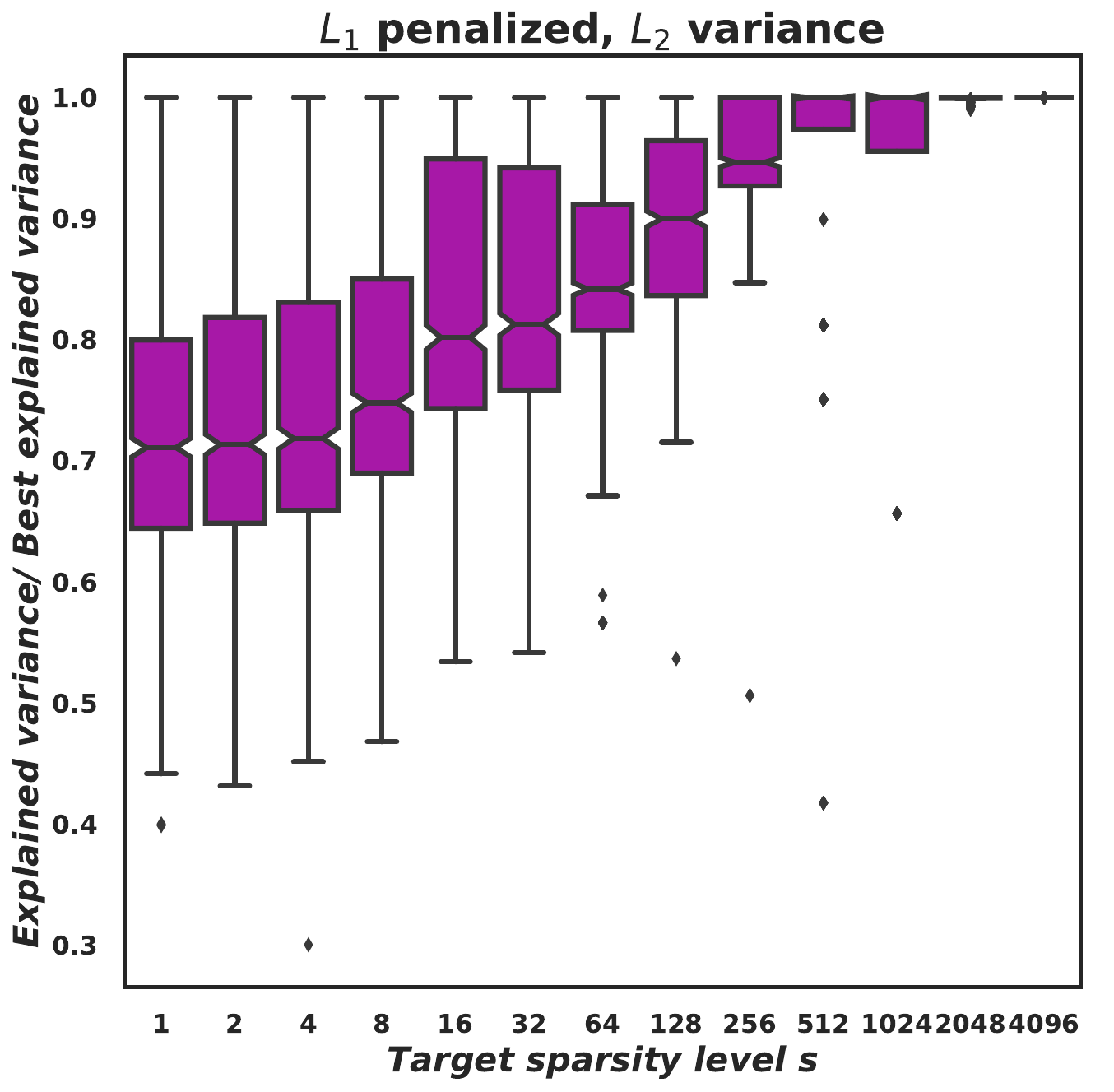}
 
  \includegraphics[scale = 0.2]{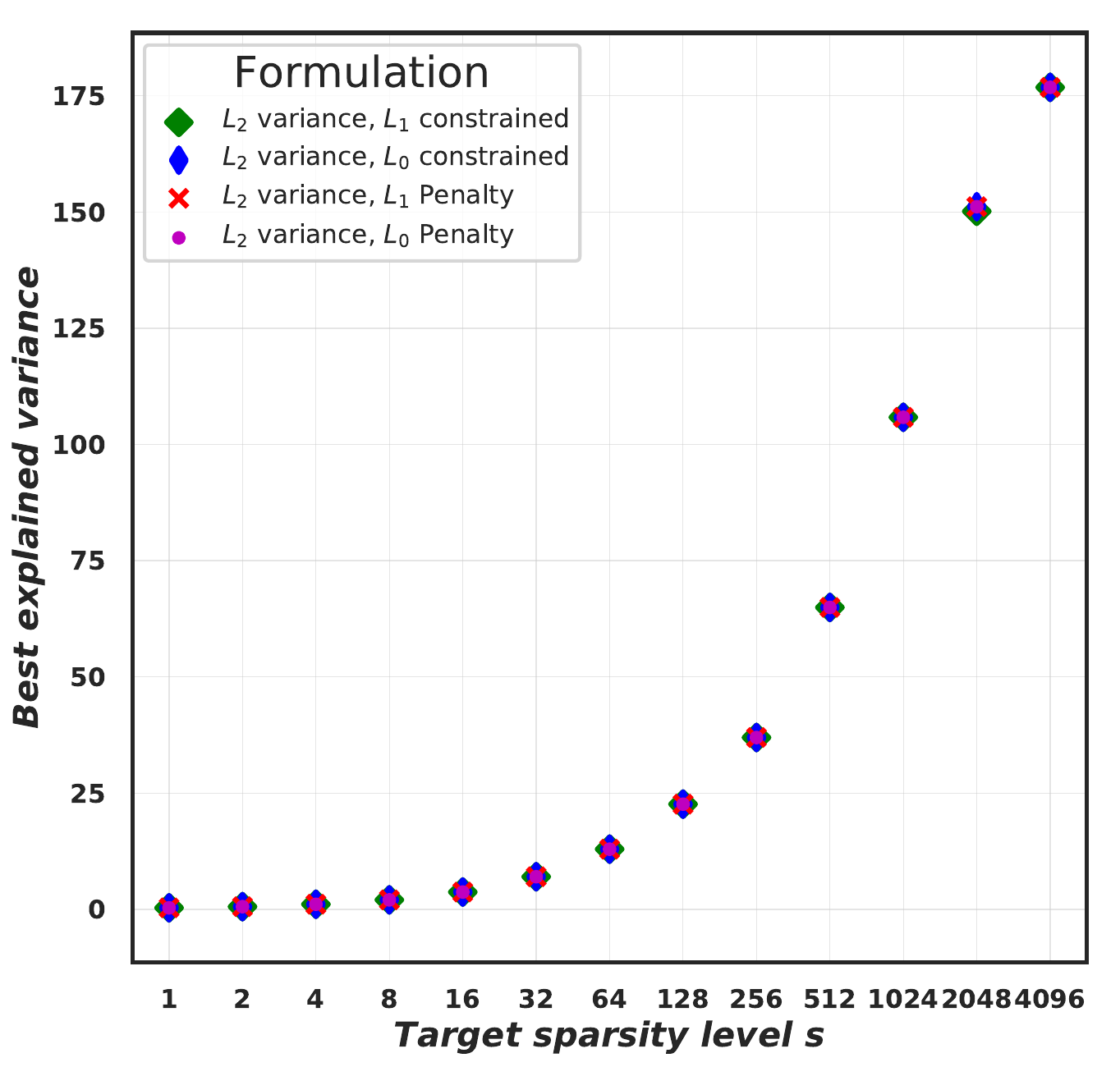}
  \includegraphics[scale = 0.2]{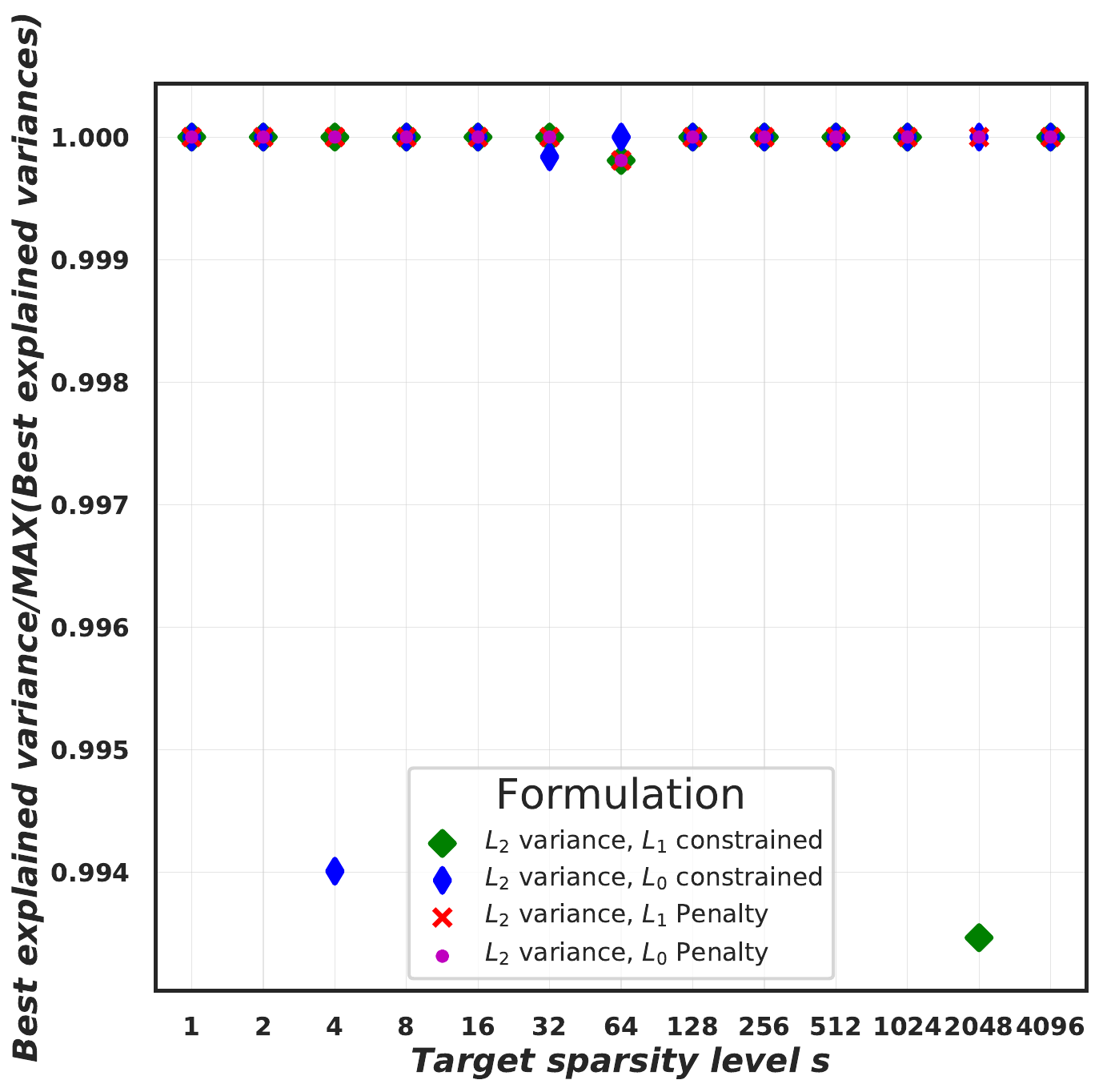}

 \caption{It may be easy to converge to a poor  solution (corresponding $L_2$ variance SPCA problems) for AT\&T Database of Faces.}
 \label{fig:BPmsL_2}
\end{figure}
As shown by \citet{JNRS10} and \citet{LT11} for GPower, and due to our equivalence theorem (Theorem 1), we know that Algorithm~2 (AM) is only able to converge to \remove{a local solution} \add{a stationary point rather than a global solution}.  Moreover, quality of the solution will depend on the starting point (SP) $x^{(0)}$ used. When the algorithm is run just once, the quality of the obtained solution, in terms of the objective value (or explained variance), can be poor. Hence, if the amount of explained variance is important, it will be useful to run the method repeatedly from a number of different SPs. \remove{In this and all subsequent experiments we  generated $A\in\RR^{n\times p}$ with independent and uniformly distributed entries from $[-1,1]$. Here chose $n=512$ and $p=2,048$ (and renormalized the columns so that their norms are uniformly distributed on $[0,1]$)}
\add{We considered ``AT\&T Database of Faces" data set\footnote{\href{https://www.kaggle.com/kasikrit/att-database-of-faces/data}{https://www.kaggle.com/kasikrit/att-database-of-faces/data}}, which contains 400 images, and the size of each image is 92x112 pixels. After reshaping the data set, the data matrix has 400 rows and 10304 columns. We normalized each row of the matrix, and centralized each column of the normalized matrix
and solved the corresponding SPCA problems described in Table \ref{tab:8REFORM} 
% $L_0$ constrained $L_2$ variance SPCA problem 
with $s = 1,2,4,\dots, 4096$}. For each $s$ we  run AM from $l=1,000$ randomly generated SPs with $maxIt=200$ and  $tol=10^{-6}$. %The results are given in Figure~\ref{fig:BPms1}, where the vertical axis corresponds to the amount of explained variance of a particular solution compared to the best solution found.
\add{It is noteworthy to mention that the explained variance for the cases with $L_2$ and $L_1$ variance are considered as $\|Ax\|^2_2$ and $\|Ax\|_1$, respectively. The results are given in Figures~\ref{fig:BPmsL_2} and \ref{fig:BPmsL_1}. In the first two rows of Figures \ref{fig:BPmsL_2} and \ref{fig:BPmsL_1}, the vertical axis corresponds to the amount of explained variance of a particular solution compared to the best solution found with respect to the target sparsity level (horizontal axis) with the above setting. For the cases with $L_1$ constrained, we considered $\lambda_s(a)$ to be updated for some predefined iterations (let's say 10), and $\lambda_s(a)$ would be fixed afterwards in order to have a stable $F(x,y)$. The same trick can be applied to the cases with penalty (cases 5-8). That is, we control sparsity level by $\gamma$ for some predefined iterations; to do so, in order to reach the sparsity level of $s$, first, we sort the vector ``$a$" based on its squared and absolute value for the operators $U_{\gamma}(a)$ and $V_{\gamma}(a)$, respectively. Then, we set $\gamma$ to be the $s^{\text{th}}$ element of the new sorted vector, and by doing so we can guarantee the sparsity level of the output vector to be $s$ for the predefined iterations, and we make $\gamma$ fixed afterwards. Overall, it means that there is no need to tune $\gamma$ in the aforementioned cases. In the third rows of the Figures \ref{fig:BPmsL_2} and \ref{fig:BPmsL_1}, the left ones show the best explained variance for the formulations with $L_2$ and $L_1$ variances among 1000 runs; and the right ones highlight that the best explained variance for all formulations are close to each other.}
% The left figures in the third rows of the Figures \ref{fig:BPmsL_2} and \ref{fig:BPmsL_1} show the best explained variance for the formulations with $L_2$ and $L_1$ variances among 1000 runs. The right figures in the third rows of the Figures \ref{fig:BPmsL_2} and \ref{fig:BPmsL_1} highlight that the best explained variance for all formulations are close to each other.}

% case L_1 Real Dataset
\begin{figure}[!ht]
 \centering
 \includegraphics[scale = 0.2]{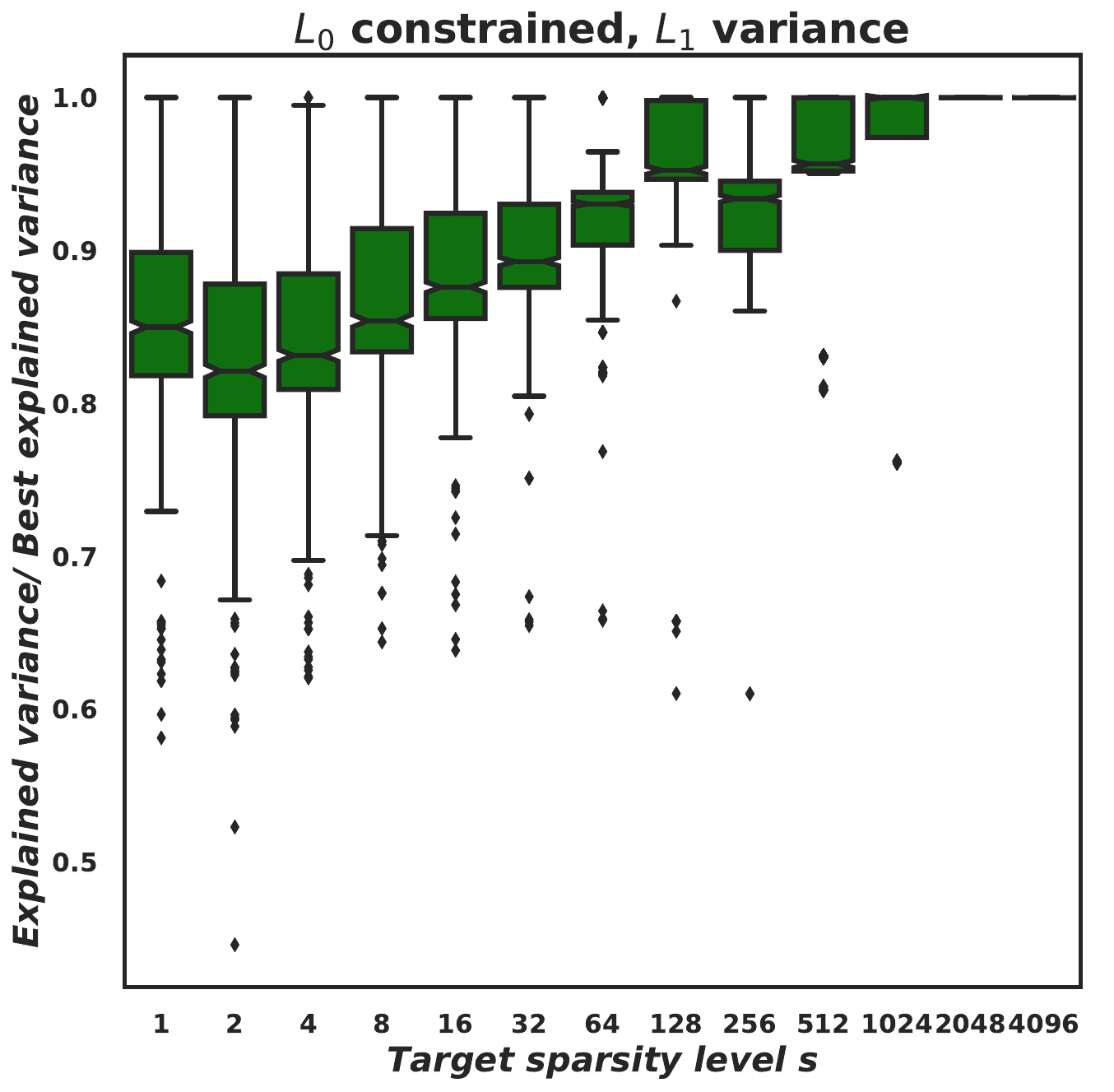}
 \includegraphics[scale = 0.2]{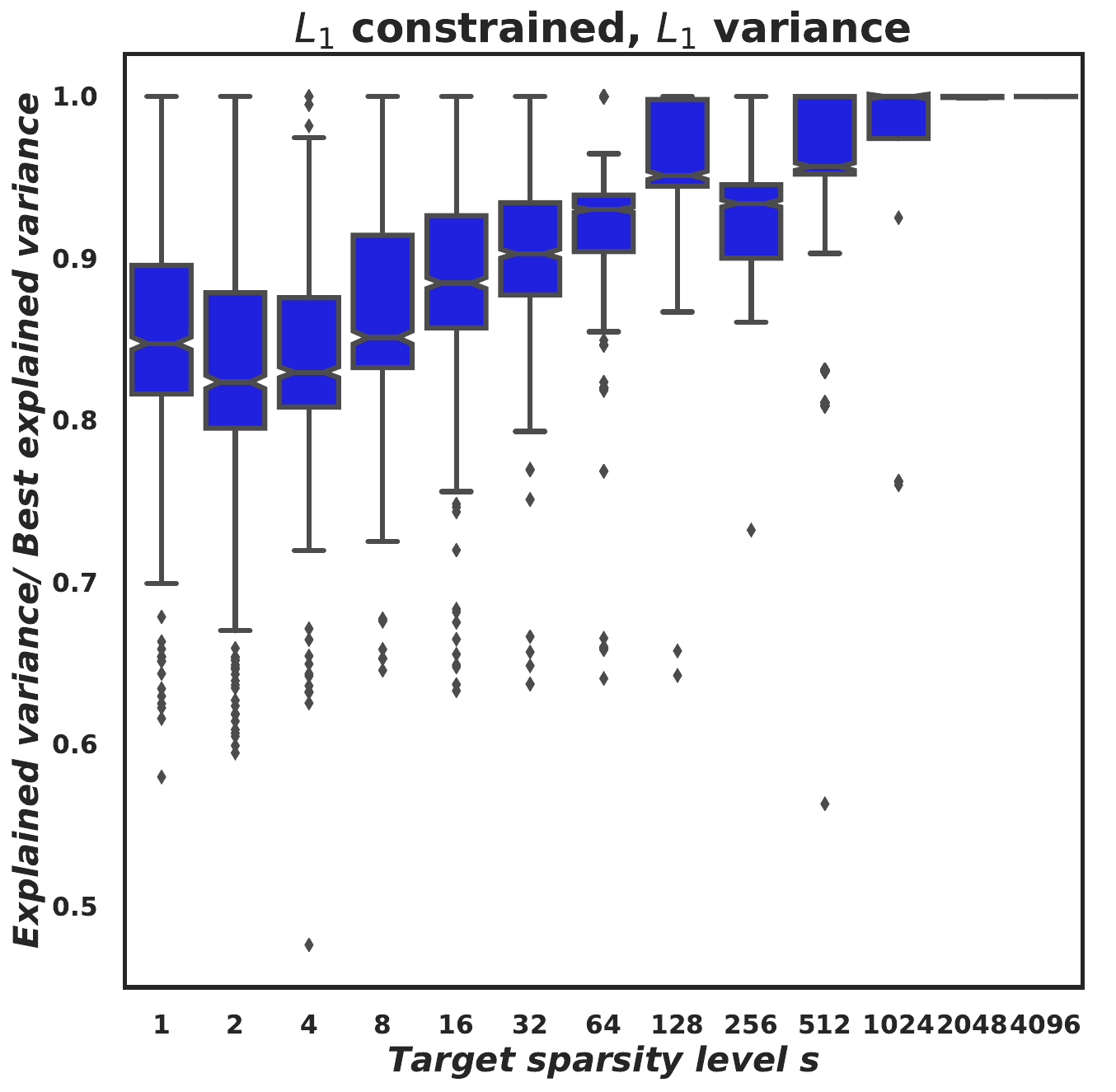}
  
 \includegraphics[scale = 0.2]{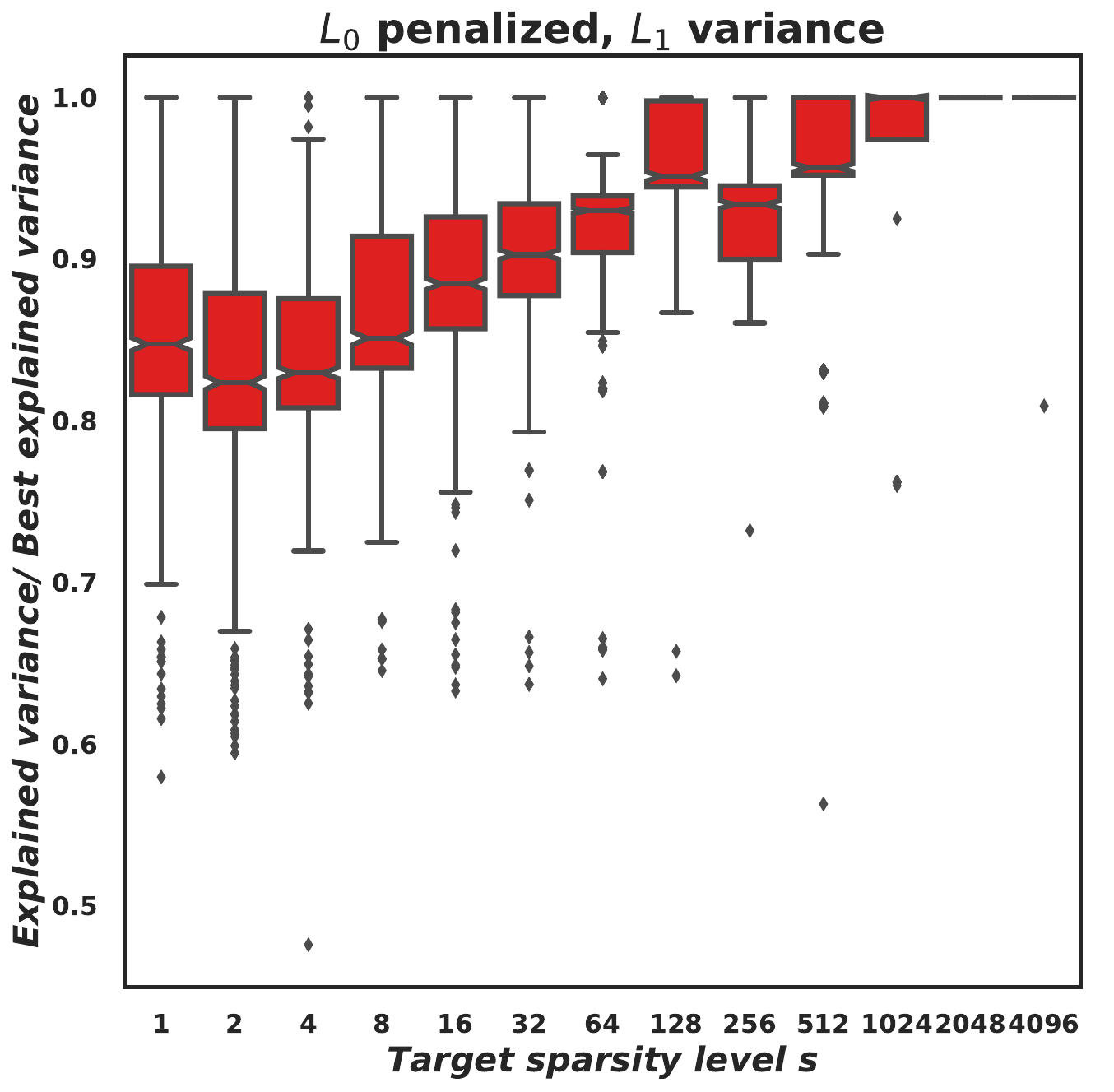}
 \includegraphics[scale = 0.2]{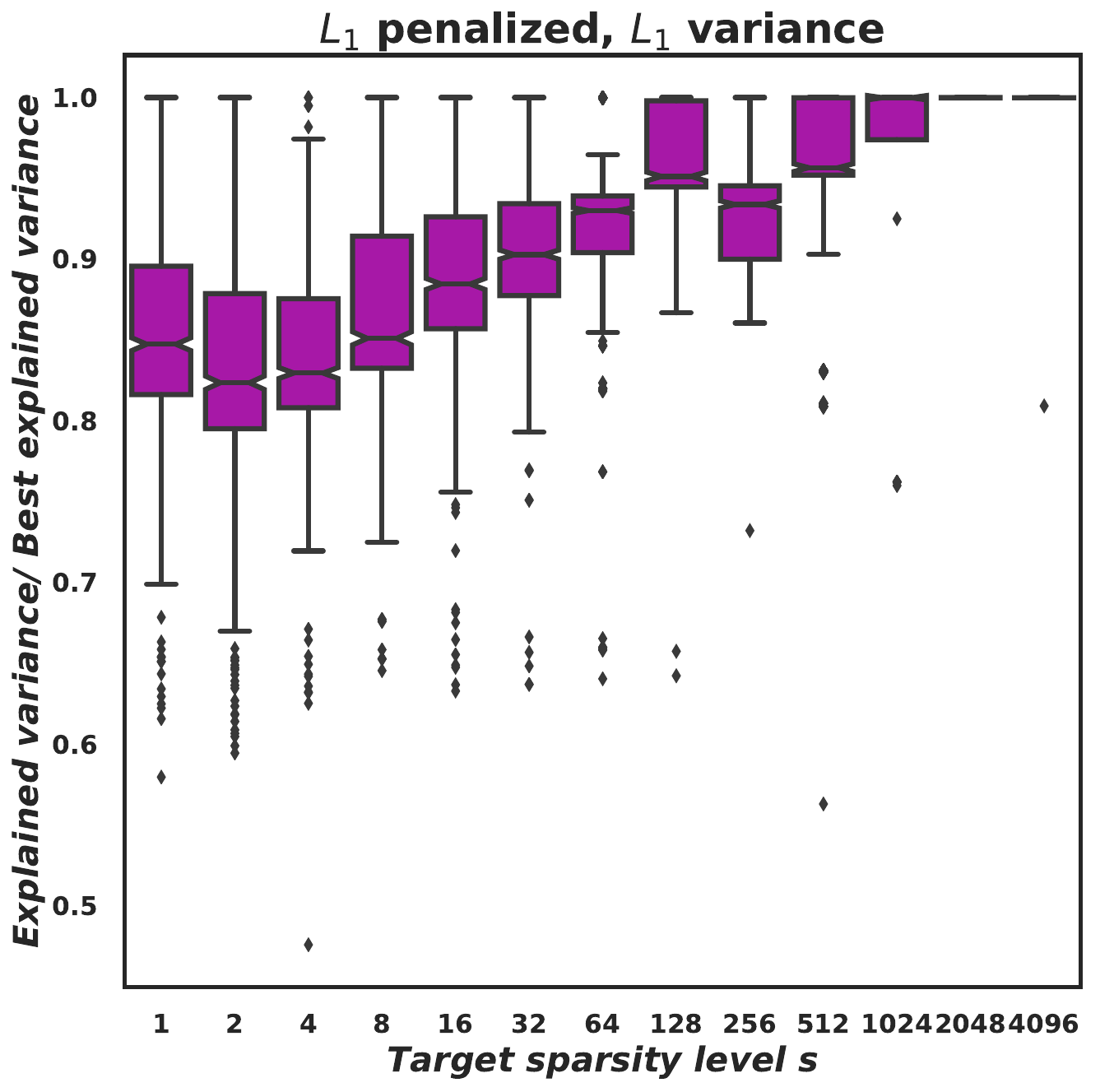}
   \includegraphics[scale = 0.2]{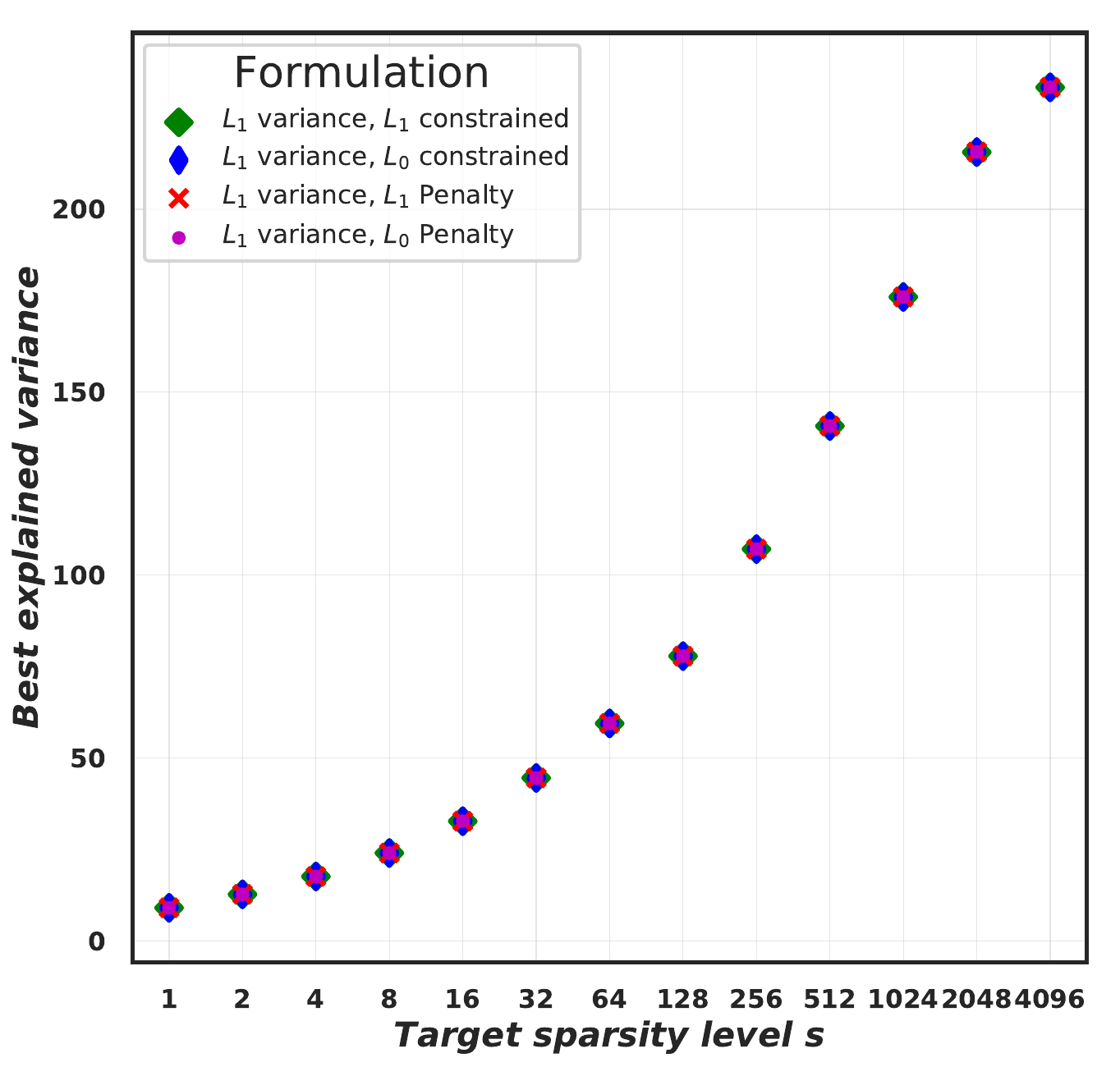}
  \includegraphics[scale = 0.2]{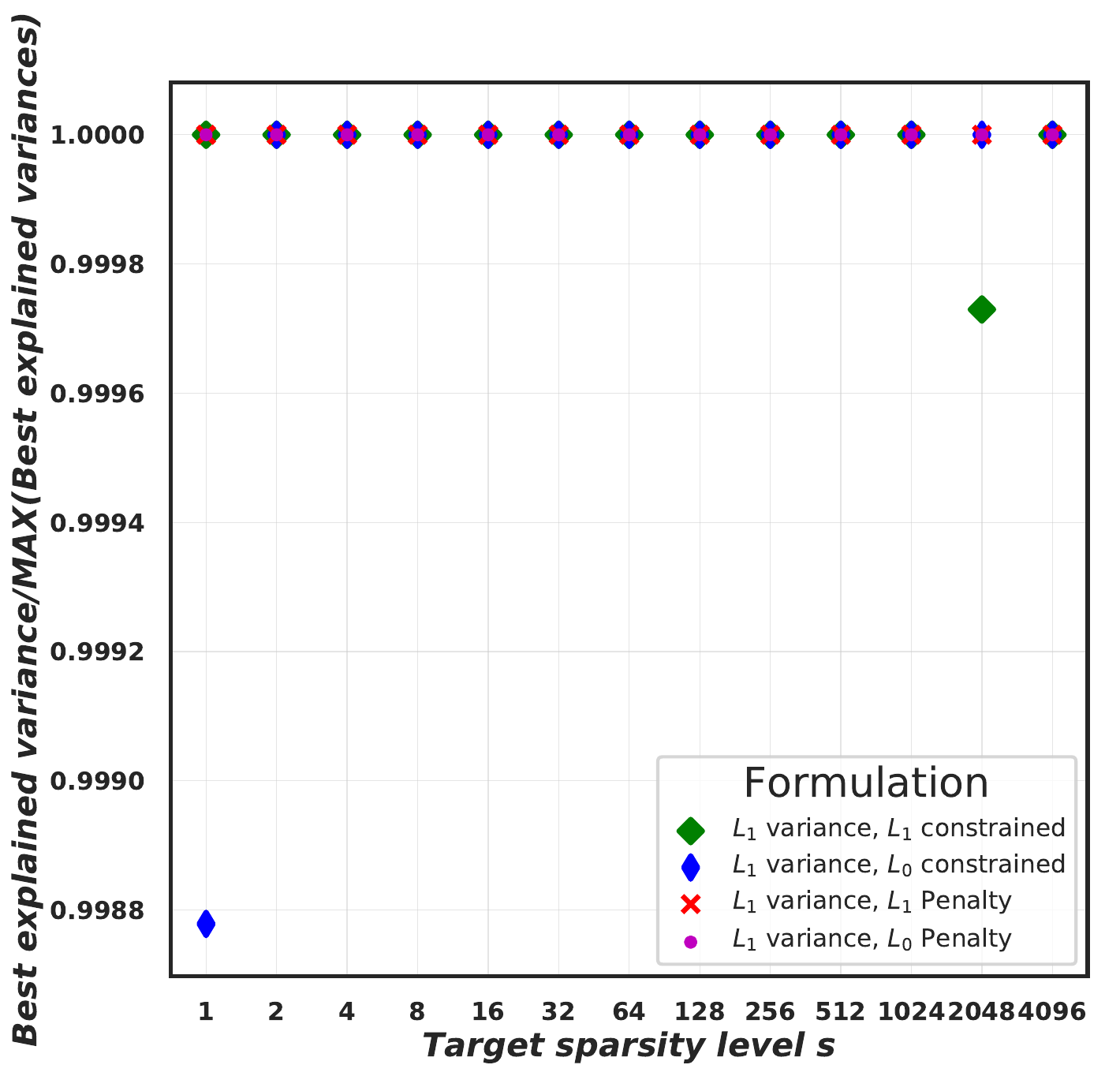}

 \caption{It may be easy to converge to a poor  solution (corresponding $L_1$ variance SPCA problems) for AT\&T Database of Faces.}
 \label{fig:BPmsL_1}
\end{figure}

Clearly, for small $s$ it is easy to obtain a bad solution if we run the method only a few times; this effect is milder for large $s$ but may be substantial nevertheless in real life problems. Hence, especially when $s$ is small,  it is necessary to employ a globalization strategy such as rerunning AM from a number of different starting points. This experiment illustrates that the simple strategy of running the method from a number of randomly generated starting points
can be effective in finding solutions with more explained variance. A ``naive'' (NAI) approach would be to do this sequentially: solve the problem with one starting point first before solving it for another starting point.

%As discussed in Sections 2 and 3, it is more efficient to employ a parallelization strategy here as well. In the next experiment we look at 2 simplest of these: i) ``start-from-all'' (SFA), in which we run AM from all starting points in parallel and ii) ii) ``batches'' (BAT) in which the starting points are divided into batches, with all SPs within a batch treated in parallel, but different batches approached sequentially.

\subsection{Economies of scale} Running AM in parallel, started from a number of SPs,  increases the utilization of computer resources, especially on parallel architectures. In order to demonstrate this, we generated 6 data matrices with $p=1000, 2000, \dots, 32000$ and  run the AM method for the $L_0$ penalized $L_2$ variance SPCA formulation with $l=256$ SPs (and $maxIt=10$). By BAT$r$ we denote the approach with batches of size $r$. Hence, SFA = BAT$256$ and NAI = BAT$1$. Besides these two basic choices, we look at BAT$4$, BAT$16$ and BAT$64$ as well. The results can be found in Figure~\ref{fig:MAIN}.

\begin{figure}[!htp]
\centering
\noindent
\includegraphics[width=0.48\linewidth]{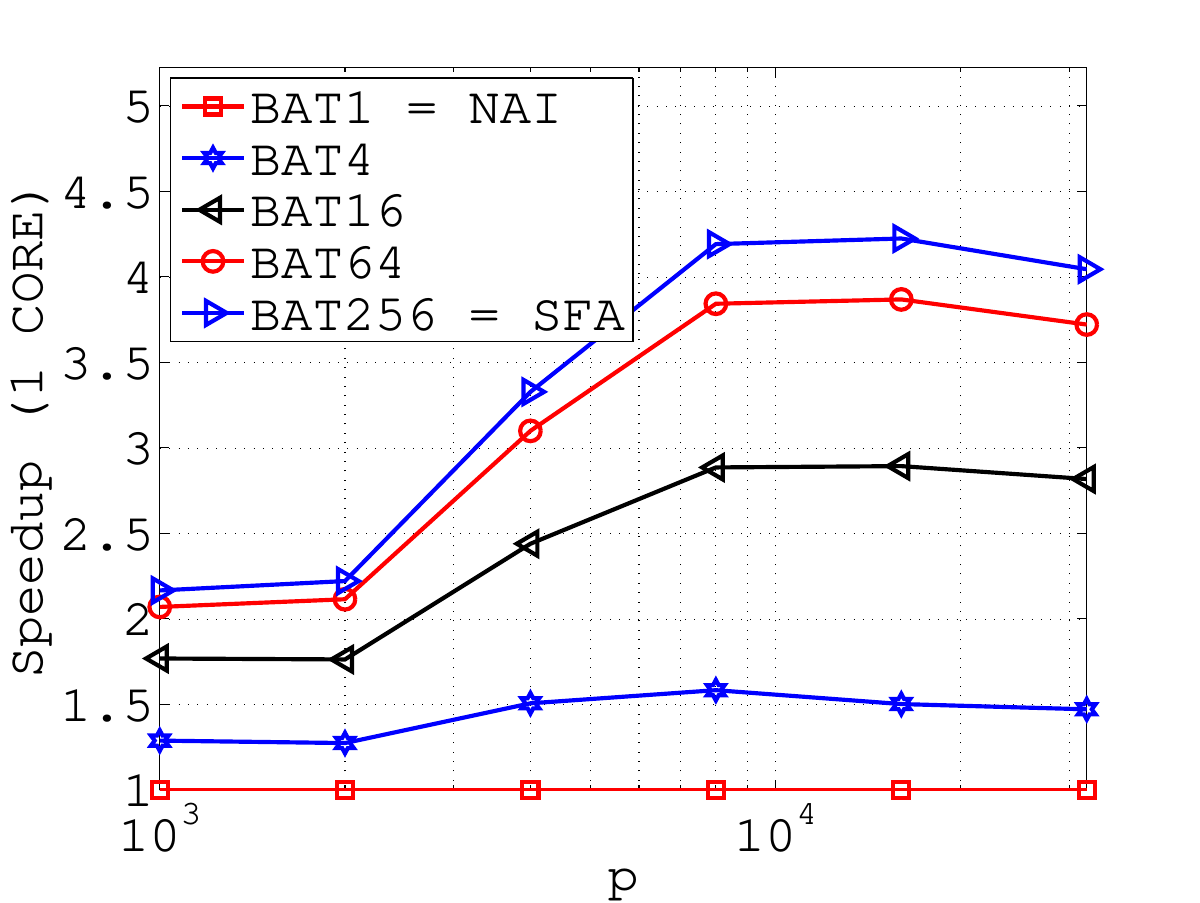}
\includegraphics[width=0.48\linewidth]{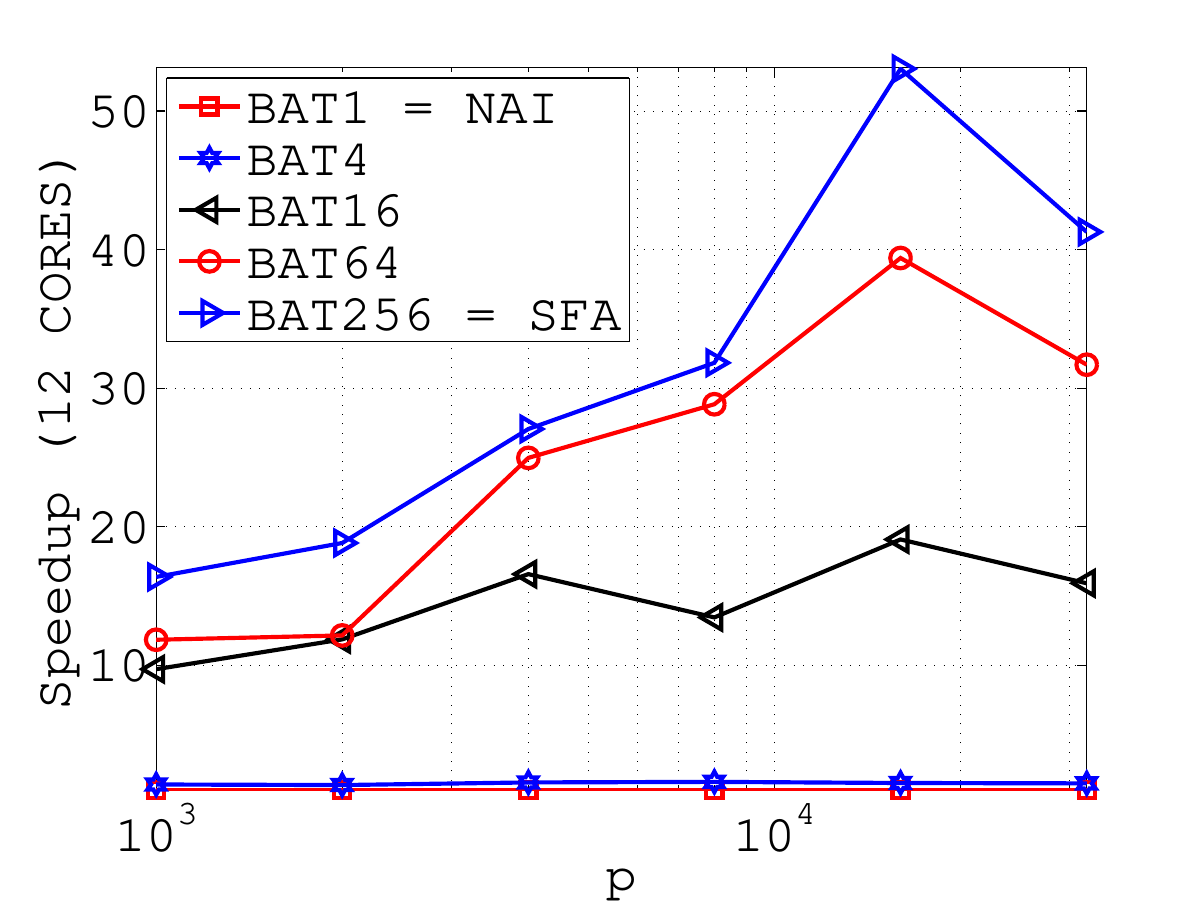}
 \caption{Economies of scale: ``Start-from-all'' (SFA) is better than any of the batching strategies on a single-core machine (LEFT); even more so on a multi-core machine (RIGHT).}
 \label{fig:MAIN}
\end{figure}

Different problem sizes $p$ appear on the horizontal axis; on the vertical axis we plot the speedup obtained by applying a particular batching strategy compared to NAI. Note that even on a single-core computer (LEFT plot) we benefit from running the methods in parallel (``economies of scale'') rather than running them one after another. Indeed, we can obtain a $2-3\times$ speedup with BAT$16$ across the whole range of problem sizes, and $4\times$ speedup with SFA for large enough $p$. With $12$ cores (RIGHT plot) the effect is much more dramatic: the speedup for BAT$16$ is consistently in the $10-20\times$ range, and can even reach $50\times$ for SFA.

\subsection{Dynamic replacement} It often happens, especially when batch size is large, that some problems within a batch converge sooner than others. The vanilla BAT approach described above does nothing about it, and continues through matrix-matrix multiplies, updating the already converged iterates, until the last problem in the batch converges. A minor but not negligible speedup is possible by employing an ``on-the-fly'' (OTF) dynamic replacement technique, where whenever a certain problem converges, it is replaced by a new one. Hence, no predefined batches exist---OTF can be viewed as a greedy list scheduling heuristic. We used $l=1024$ starting points and compare SFA$1024$ with BAT$64$ and OTF$64$--the dynamic replacement variant of BAT$64$.

\begin{figure}[!htp]
\centering
\noindent
 \includegraphics[width=0.48\linewidth]{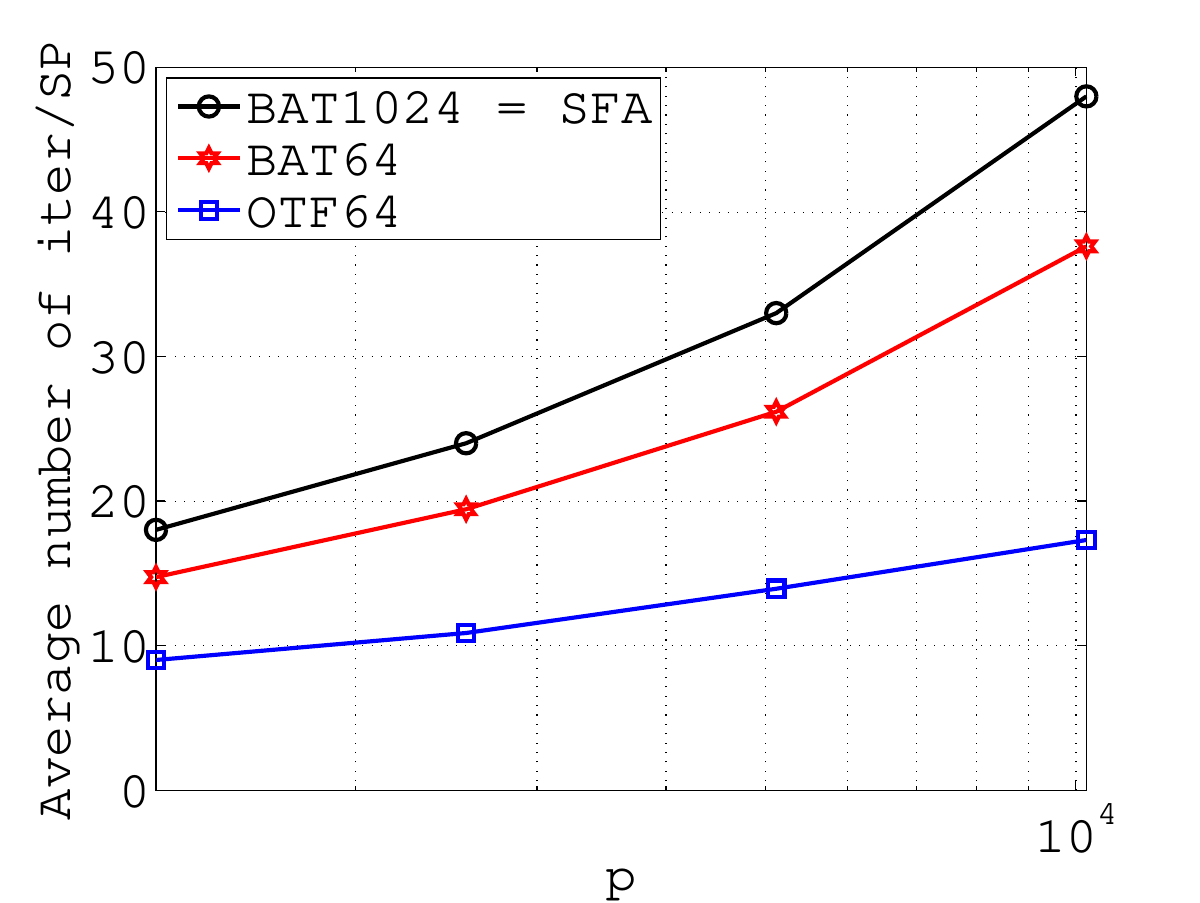}
 \includegraphics[width=0.48\linewidth]{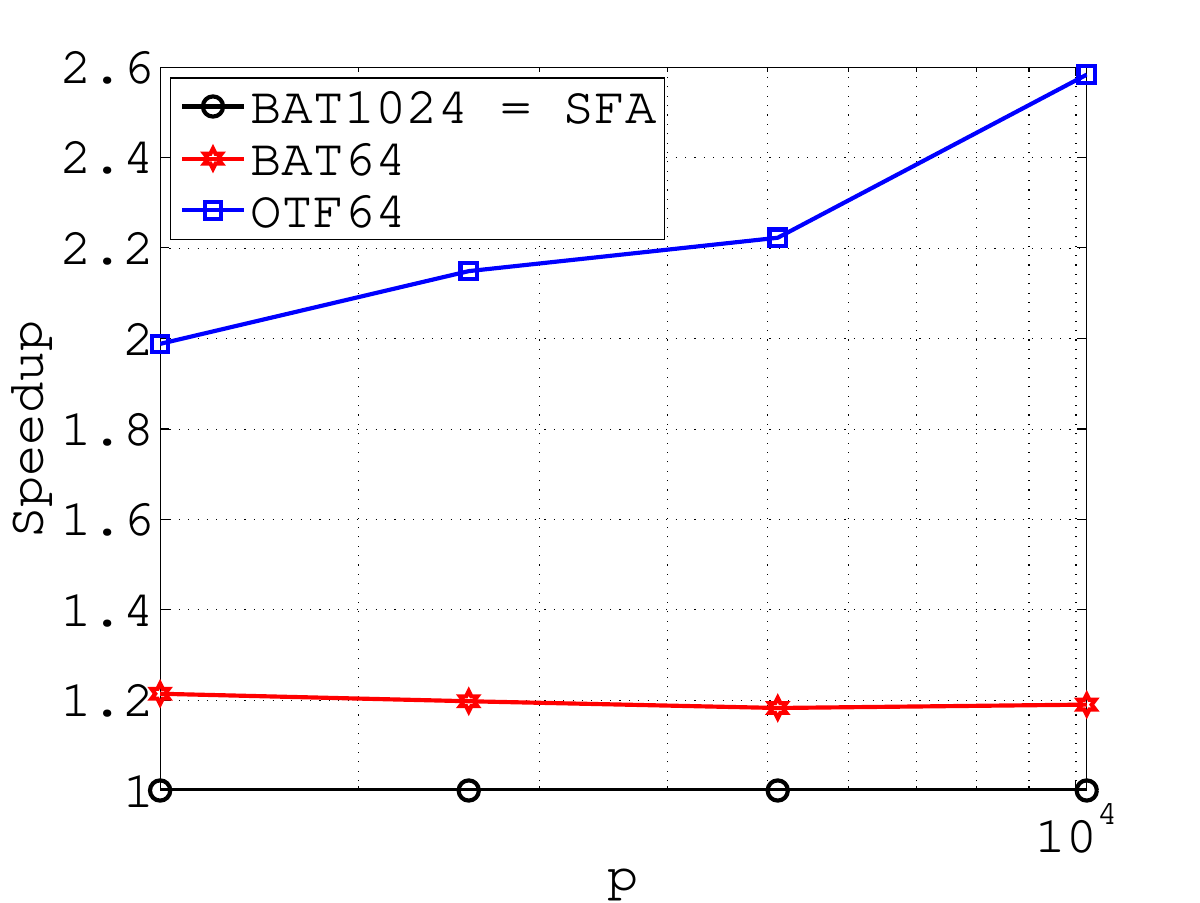}
 \caption{Dynamic Replacement: ``On-the-fly'' (OTF) is better than ``Batches'' (BAT), which is better than ``start-from-all'' (SFA).}
 \label{fig:MAIN2}
\end{figure}

Looking at the LEFT plot in Figure~\ref{fig:MAIN2}, we see that the average number of iterations per starting point is much smaller for OTF. This results in speedup of more than $2\times$ when compared with SFA (RIGHT plot). Notably, SFA is \emph{slower} than both BAT64 and OTF64, which shows that it may not be optimal to choose $r=l$.

\section{Multi-core Processors, GPUs and Clusters}\label{sec:arch}

Accompanying this paper is the open source software package \textbf{``24am''}\footnote{
\remove{\href{https://code.google.com/p/24am/}{https://code.google.com/p/24am/}}
\add{\href{https://github.com/optml/24am}{https://github.com/optml/24am}}.}  implementing parallelization strategies described in Section~\ref{sec:par},  all with Algorithm 2 (AM) used as the underlying solution method, with the option of using any of the 8 optimization formulations of SPCA described in Table~\ref{tab:8main}. The name 24am comes from the fact that we implement the solver for 3 different parallel architectures: multi-core processors, GPUs and computer clusters, leading to $24=8\times 3$ methods based on AM.

\add{In the rest of this section}
 we first perform several numerical experiments illustrating the speedups obtained by parallelization on these three computing architectures. We then conclude with a real-life numerical example (large text corpora) and a few implementation remarks.

\subsection{Multi-core speedup} Here we solve 9 random $L_1$ constrained $L_1$ variance SPCA instances of sizes $p=100\times 2^i$, $i=1,\dots,9$, $n=p/10$, with $100$ SPs each, on a machine using $1, 2, 4$ and $8$ cores; see Figure~\ref{fig:MAIN3}.

\begin{figure}[!htp]
\centering
\noindent
 \includegraphics[width=0.48\linewidth]{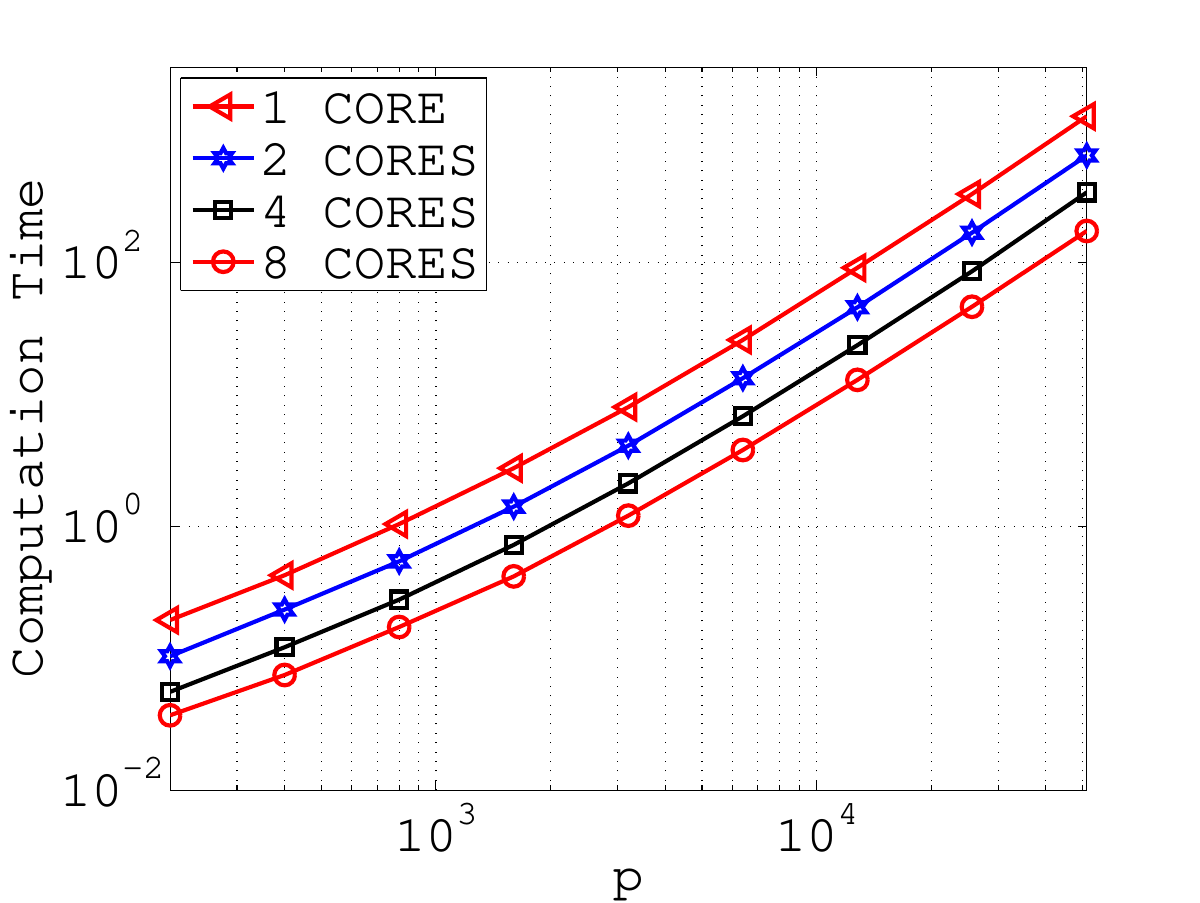}
 \includegraphics[width=0.48\linewidth]{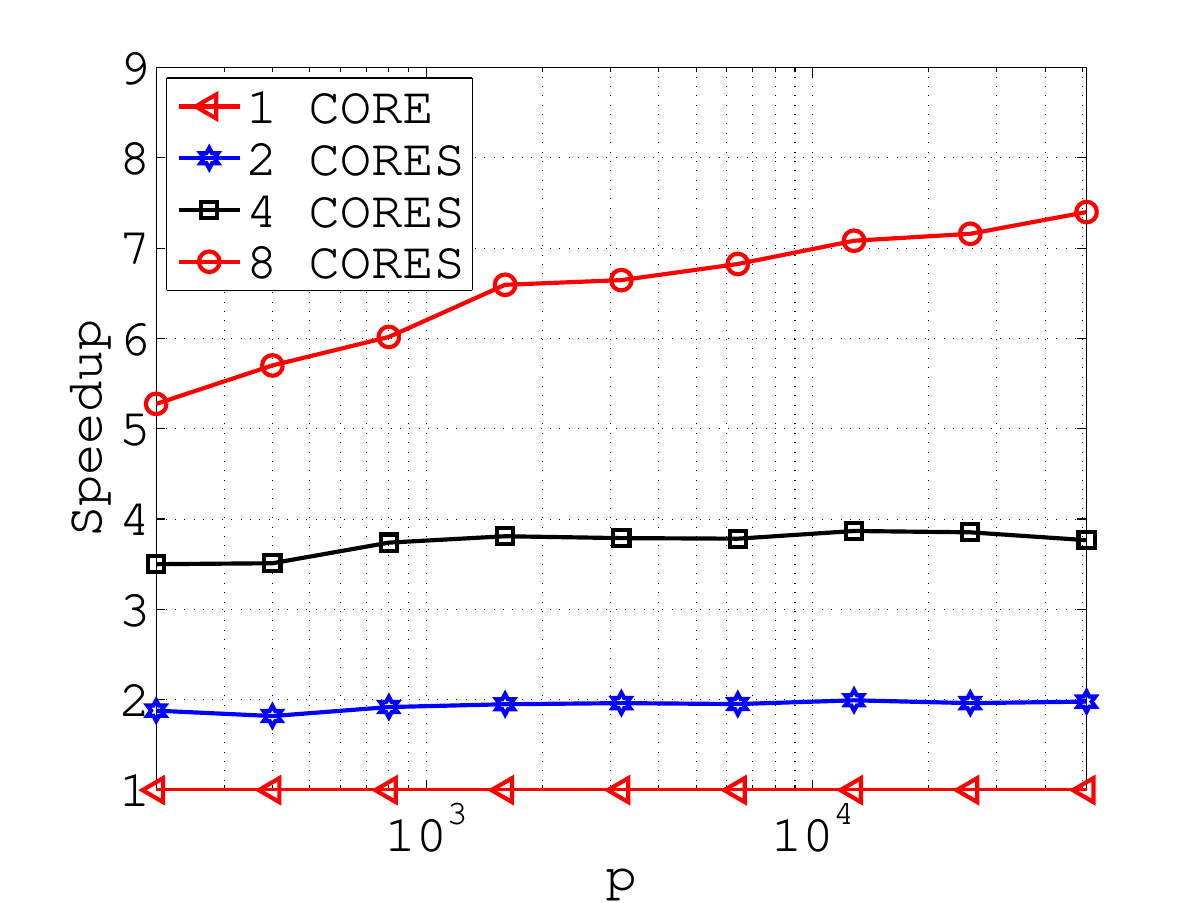}
 \caption{Multi-core speedup is proportional to the number of cores.}
 \label{fig:MAIN3}
\end{figure}

 The plot on the LEFT shows the total computational time; the plot on the RIGHT shows the speedup of multi-core codes compared to the single-core code. Note that the speedup is consistently close to the number of cores for the 2 and 4-core setups across all problem sizes, and is growing with $p$ from $5\times$ to about $7.5 \times$ in the $8$-core setup.

\subsection{GPU speedup} Here we solve  8 random $L_1$ penalized $L_1$ variance SPCA instances with $p$ varying roughly between $10^3$ and $10^5$, and $n=p/200$. We  solved all formulations with $\{1,16,256\}$ SPs on a single-core CPU and a GPU; the results are shown in Figure~\ref{fig:MAIN4}.

\begin{figure}[!htp]
\centering
\noindent
 \includegraphics[width=0.48\linewidth]{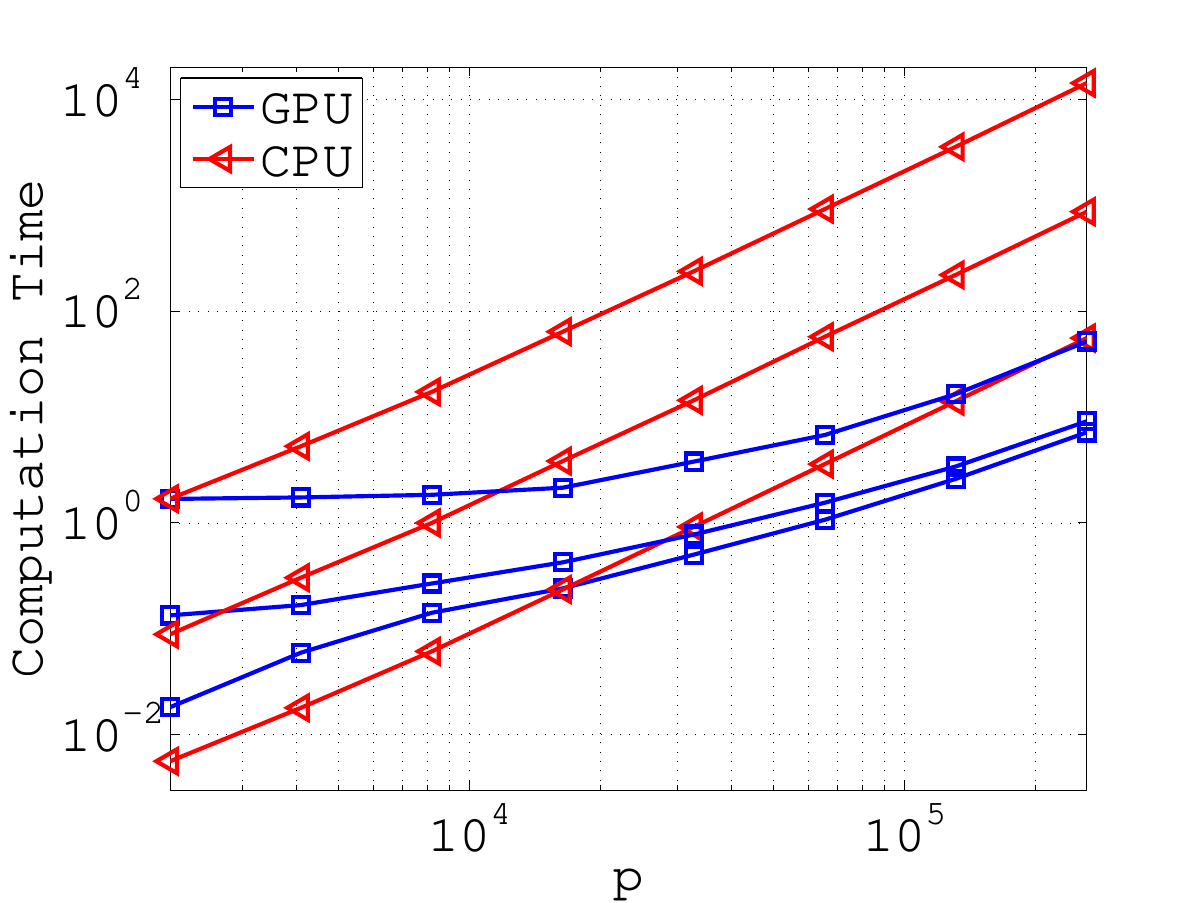}
 \includegraphics[width=0.48\linewidth]{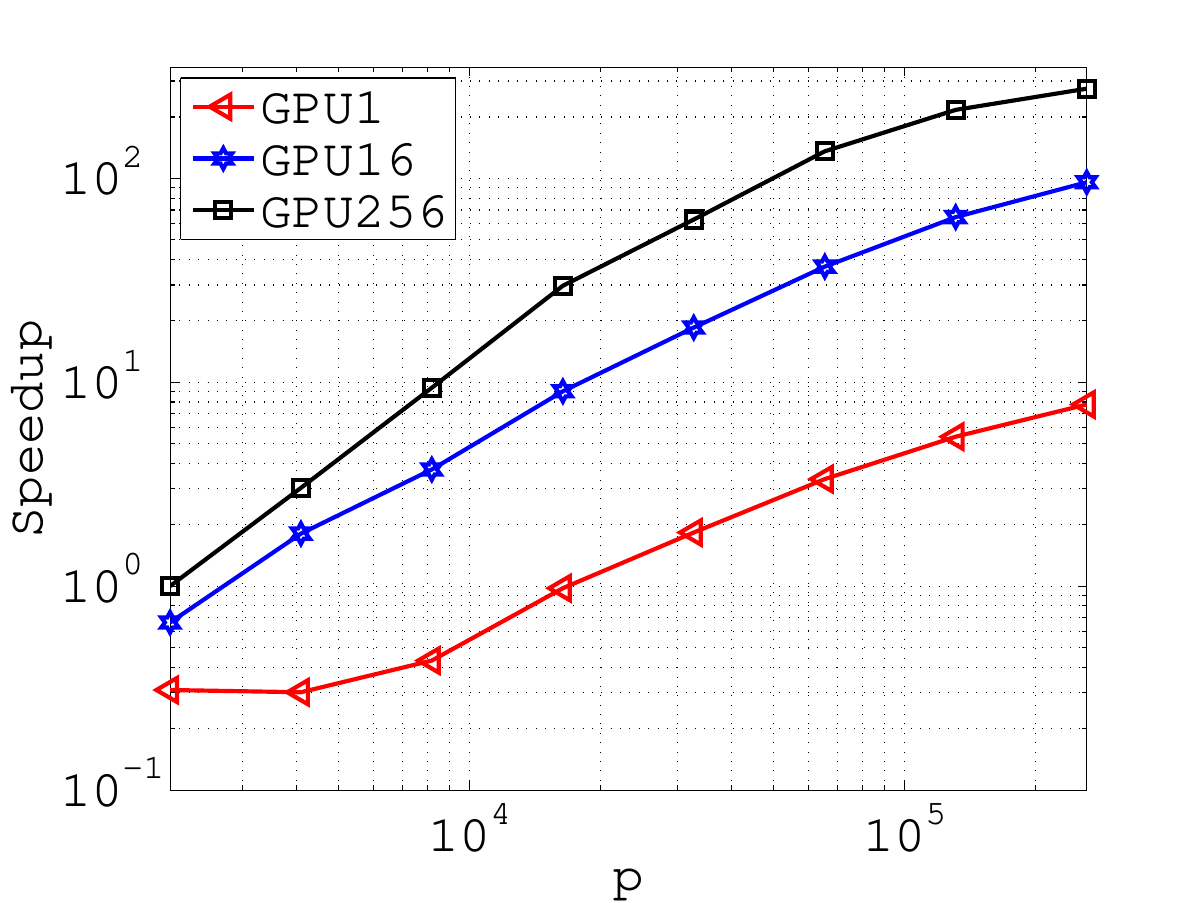}
 \caption{GPU code can achieve $125\times$ speedup compared to single-core when 256 starting points are used.}
 \label{fig:MAIN4}
\end{figure}

The plot on the LEFT shows the total computational time. The red lines with triangle markers correspond to the single-core setup, the ``higher'' the line, the more starting points were used. The blue lines with square markers correspond to our GPU codes. While the runtime increases linearly with problem size for the single-core codes, it grows slowly for the GPU codes. Note that the GPU code may actually be slower for small problem sizes. Looking at the RIGHT plot, we see that the GPU code is capable of a $100$-$125\times$ speedup; this happens for large problem sizes and $256$ SPs. The speedup can reach $100 \times$ for $16$ SPs as well.

\subsection{Cluster code} In this experiment we  solved several $L_1$ penalized $L_2$ variance SPCA problems with a \emph{fully dense} matrix $A\in \RR^{n\times p}$; the results are in Table~\ref{tbl:clusterResult}. We focus our discussion on the largest of the problems only (last three lines of the table), one with
   $n=6\times 10^3$ and $p=8\times 10^6$. We used a cluster of 800 CPUs; storage of the data matrix required 357.6 GB of memory. The matrix was first loaded from files to memory; this process took $t_1=92$ seconds. Subsequently, the loaded data was distributed to CPUs where needed, which took additional $t_2=713$ seconds. Finally we  run the AM method with $1$, $32$ and $64$ starting points and measured the average time of a single iteration; the results are $t_3^1 = 4.1$, $t_3^{1} = 51.1$ and $t_3^{1} = 134.9$ seconds, respectively. When using a single starting point, the method would converge in about a minute. The $t_3^{k}$ column of Table~\ref{tbl:clusterResult} depicts the time it takes for the solver to perform $k$ iterations. We treated the problem directly, without using any safe feature elimination  techniques \citep{ZhElGh11}. Such preprocessing could, however, be able to expand the reach of our cluster code to even larger problem sizes.

\begin{table}[!ht]
\begin{center}
\footnotesize
\tiny 
\begin{tabular}{ l r r c r r r r r r }
\toprule
$n\times p$ &  memory& \# CPUs & GRID & SP&  $t_1$ & $t_2$&$t_3^1$& $t_3^{4}$& $t_3^{16}$ %& $\sum t_i$
\\ \midrule
$10^4\times 2\cdot10^5$ & 14.9 GB & 20 & $10 \times 2$&1 &42.68 & 0.86 &  0.56&2.06&8.48 % &-
\\
$10^4\times 2\cdot10^5$ & 14.9 GB & 20 & $10 \times 2$&32 &-&-&4.60&18.89&87.84 %&-
\\
$10^4\times 2\cdot10^5$ & 14.9 GB & 20 & $10 \times 2$&64&-&-&10.47&37.88&166.60 %&517.40
\\ \hdashline
$6\cdot 10^3\times 4\cdot10^5$ & 17.8 GB & 40 & $10 \times 4$&1&26.89&86.33&0.78&3.15&9.96%&-
\\
$6\cdot 10^3\times 4\cdot10^5$ & 17.8 GB & 40 & $10 \times 4$&32&-&-&7.39&27.72&125.14%&-
\\
$6\cdot 10^3\times 4\cdot10^5$ & 17.8 GB & 40 & $10 \times 4$&64&-&-&13.19&58.36&201.51%&681.75
\\ \hdashline
$6\cdot 10^3\times 10^6$ & 44.7 GB & 100 & $10 \times 10$&1&49.22&104.26&0.45&2.44&11.62%&-
\\
$6\cdot 10^3\times 10^6$ & 44.7 GB & 100 & $10 \times 10$&32&-&-&6.37&29.72&115.73%&-
\\
$6\cdot 10^3\times 10^6$ & 44.7 GB & 100 & $10 \times 10$&64&-&-&14.14&52.64&219.8%&718.89
\\ \hdashline
$6\cdot 10^3\times 4\cdot10^6$ & 178.8 GB & 400 & $10 \times 40$&1&129.69&611.69&1.24&5.12&31.46%&-
\\
$6\cdot 10^3\times 4\cdot10^6$ & 178.8 GB & 400 & $10 \times 40$&32&-&-&17.50&61.36&255.80%&-
\\
$6\cdot 10^3\times 4\cdot10^6$ & 178.8 GB & 400 & $10 \times 40$&64&-&-&31.36&141.61&525.08%&-
\\ \hdashline
$6\cdot 10^3\times 8\cdot10^6$ & 357.6  GB & 800 & $10 \times 80$&1& 92.12&713.45&4.14&15.82&95.51%&-
\\
$6\cdot 10^3\times 8\cdot10^6$ & 357.6  GB & 800 & $10 \times 80$&32&-&-&51.11& 324.26&619.45%&-
\\
$6\cdot 10^3\times 8\cdot10^6$ &  357.6 GB & 800 & $10 \times 80$&64&-&-& 134.89&690.06&-%&-
\\
\bottomrule
\end{tabular}
\end{center}

\caption{Experiments with the cluster implementation. For the first experiment (first three rows)
the dimensions of the virtual grid matched the size of the data matrix, whence $t_2$ is small.
%For the rest we used a grid consisting of $64\times 64$ submatrices.
}
\label{tbl:clusterResult}
\end{table}

\begin{table}[!ht]
\begin{center}
\footnotesize
\begin{tabular}{ c c c c c }
\toprule
NYT 1$^{st}$ PC & NYT 2$^{nd}$ PC & NYT 3$^{rd}$ PC & NYT 4$^{th}$ PC & NYT 5$^{th}$ PC \\ \midrule
game & companies & campaign & children & attack \\\hdashline 
play & company & president & program & government \\\hdashline 
player & million & al gore & school & official\\\hdashline 
season & percent & bush & student & US\\\hdashline 
team & stock & george bush & teacher & united states \\
\bottomrule
\end{tabular}

\vskip5pt 

\begin{tabular}{ c c c c c }
\toprule
PubMed 1$^{st}$ PC & PubMed $2^{nd}$ PC & PubMed $3^{rd}$ PC & PubMed 4$^{th}$ PC & PubMed 5$^{th}$ PC \\
\midrule
disease & cell & activity & cancer & age \\\hdashline 
level & effect & concentration & malignant & child \\\hdashline 
patient & expression & control & mice & children \\\hdashline 
therapy & human & rat & primary & parent \\\hdashline 
treatment & protein & receptor & tumor & year\\
\bottomrule
\end{tabular}
\end{center}
\caption{First 5 sparse PCs for NYTimes and PubMed data sets.}\label{tab:LTC}
\end{table}

\subsection{Large text corpora} In the first experiment we tested the AM method with $L_0$ constrained $L_2$ variance formulation (with $s=5$) on two medium-size data sets from the \emph{Machine Learning Repository}\footnote{\href{http://archive.ics.uci.edu/ml/datasets/Bag+of+Words}{http://archive.ics.uci.edu/ml/datasets/Bag+of+Words}}: news articles appeared in New York Times and  abstracts of articles published in PubMed. Each data set is formatted as a matrix $A\in\RR^{n\times p}$, where the rows of $A$ correspond to news articles in the NYTimes data set and to abstracts in PubMed, and the columns correspond to words. The number of appearances of word $j$ in article or abstract $i$ is the $(i,j)$-th entry of $A$; the matrices are hence clearly sparse. The NYTimes data set has 300,000 articles, 102,660 words, and approximately 70 million nonzero entries. The PubMed data set contains 8.2 million articles, 141,043 words, and approximately 484 million nonzeroes. The matrices can be stored in 0.778 GB and 5.42 GB memory space, respectively.  We have customized the AM method to exploit sparsity as much as possible. In Table~\ref{tab:LTC} we present the first 5 sparse principal components (5 words each). Clearly, the first PC for NYT is about sports, the second about business, the third about elections, the fourth about education and the fifth about United States. Similar interpretations can be given to the PubMed PCs. \add{We also tested the AM method with other formulations reported in Table \ref{tab:8REFORM} for the NYTimes data set. Table \ref{tab:LTC_L2} illustrates the first 5 sparse principal components regarding the formulations with $L_2$ variance\footnote{Note that the different colors in tables \ref{tab:LTC_L2} and \ref{tab:LTC_L1} are corresponding to the formulations with the same color in Table \ref{tab:8REFORM}.}. We also provided the nonzero values of sparse principal components corresponding to each word, and sort each principal component based on the values for each word. Furthermore, Table \ref{tab:LTC_L1} presents the first 5 sparse principal components regarding the formulations with $L_1$ variance. For each formulation, we ran AM method by starting from $l = 20$ random starting points with $maxIt = 20$ and $tol = 10^{-6}$. Moreover, Tables \ref{tab:bestVarNYT_L2var} and \ref{tab:bestVarNYT_L1var} show the best variances (among 20 runs) with respect to the first 5 sparse PCs for the NYTimes data set for the formulation with $L_2$ and $L_1$ variances, respectively.}

\begin{table}[h!]
\begin{center}
\footnotesize 
\begin{tabular}{ c c c c c c }\toprule
 
  NYT 1$^{st}$ PC & NYT 2$^{nd}$ PC & NYT 3$^{rd}$ PC & NYT 4$^{th}$ PC & NYT 5$^{th}$ PC \\ \midrule
  \textcolor{green!70!black}{ 2000778.58}&\textcolor{green!70!black}{ 1912905.67}&\textcolor{green!70!black}{ 1560637.32}&\textcolor{green!70!black}{ 1429685.36}&\textcolor{green!70!black}{ 1193802.56}\\ 
 \hdashline
 \textcolor{blue}{ 2000778.59}&\textcolor{blue}{ 1912905.66}&\textcolor{blue}{ 1560637.45}&\textcolor{blue}{ 1429685.37}&\textcolor{blue}{ 1193803.32}\\
 \hdashline
 \textcolor{red}{2000778.60}&\textcolor{red}{1912906.01}&\textcolor{red}{1560637.21}&\textcolor{red}{1429685.37}&\textcolor{red}{1193838.99}\\
 \hdashline
 \textcolor{purple!50!blue}{ 1912905.56}&\textcolor{purple!50!blue}{ 2000778.59}&\textcolor{purple!50!blue}{ 1560636.59}&\textcolor{purple!50!blue}{ 1429685.34}&\textcolor{purple!50!blue}{ 1193792.20}\\
\bottomrule
\end{tabular}
\end{center}
\caption{
\add{The best variance w.r.t. the first 5 sparse PCs for NYTimes data set for $L_2$ variance, with $L_0$ constraint / $L_1$ constraint / $L_0$ penalty / $L_1$ penalty.}}\label{tab:bestVarNYT_L2var}
\end{table}

\begin{table}[h!]
\begin{center}
\footnotesize 
\begin{tabular}{ c c c c c c }\toprule
  NYT 1$^{st}$ PC & NYT 2$^{nd}$ PC & NYT 3$^{rd}$ PC & NYT 4$^{th}$ PC & NYT 5$^{th}$ PC \\ \midrule
\textcolor{green!70!black}{ 486843.78}&\textcolor{green!70!black}{ 462445.23}&\textcolor{green!70!black}{ 386907.51}&\textcolor{green!70!black}{ 320581.40}&\textcolor{green!70!black}{ 315784.42}\\
 \hdashline
 \textcolor{blue}{486843.78}&\textcolor{blue}{462445.23}&\textcolor{blue}{384622.40}&\textcolor{blue}{336912.52}&\textcolor{blue}{347835.82}\\
 \hdashline
 \textcolor{red}{486843.78}&\textcolor{red}{462391.75}&\textcolor{red}{387579.36}&\textcolor{red}{309628.15}&\textcolor{red}{295577.97}\\
 \hdashline
 \textcolor{purple!50!blue}{486843.78}&\textcolor{purple!50!blue}{462445.23}&\textcolor{purple!50!blue}{387901.14}&\textcolor{purple!50!blue}{319704.28}&\textcolor{purple!50!blue}{306050.47}\\
\bottomrule
\end{tabular}
\end{center}
\caption{\add{
The best variance w.r.t. the first 5 sparse PCs for NYTimes data set for $L_1$ variance, with $L_0$ constraint / $L_1$ constraint / $L_0$ penalty / $L_1$ penalty.}}\label{tab:bestVarNYT_L1var}
\end{table}

\begin{table}[h!]
\begin{center}
\footnotesize
\begin{tabular}{ c c c c c }\toprule
NYT 1$^{st}$ PC & NYT 2$^{nd}$ PC & NYT 3$^{rd}$ PC & NYT 4$^{th}$ PC & NYT 5$^{th}$ PC \\ \midrule
 \textcolor{green!70!black}{  team} &\textcolor{green!70!black}{   percent} & \textcolor{green!70!black}{  al gore }& \textcolor{green!70!black}{  school} &  \textcolor{green!70!black}{ official }\\ 
 (0.6118)&(0.6768) & (0.6115)& (0.8143)&(0.7183)\\ \hdashline 
\textcolor{green!70!black}{  game }&\textcolor{green!70!black}{   company} &  \textcolor{green!70!black}{ george bush }&\textcolor{green!70!black}{   student }& \textcolor{green!70!black}{  government }\\ 
 (0.4499)&(0.5117) & (0.4710)& (0.5139)&(0.4570)\\ \hdashline   
\textcolor{green!70!black}{  season }& \textcolor{green!70!black}{  million} & \textcolor{green!70!black}{  bush} & \textcolor{green!70!black}{  program }& \textcolor{green!70!black}{  US}\\ 
 (0.4368)&(0.3497) & (0.4539)& (0.1616)&(0.3208)\\ \hdashline 
\textcolor{green!70!black}{  player} &\textcolor{green!70!black}{   companies} & \textcolor{green!70!black}{  campaign  }& \textcolor{green!70!black}{  teacher }& \textcolor{green!70!black}{  united states}\\ 
 (0.3833)&(0.2868) & (0.3284)& (0.1549)&(0.3064)\\ \hdashline 
\textcolor{green!70!black}{  play} &\textcolor{green!70!black}{   stock }& \textcolor{green!70!black}{ president } & \textcolor{green!70!black}{  children} & \textcolor{green!70!black}{ attack }\\ 
 (0.2921)&(0.2746) & (0.3002)& (0.1499)&(0.2796)\\ 
\hline
 \textcolor{blue}{  team} &\textcolor{blue}{   percent} & \textcolor{blue}{  al gore }& \textcolor{blue}{  school} &  \textcolor{blue}{ official }\\ 
 (0.6119)&(0.6768) & (0.6123)& (0.8144)&(0.7185)\\ \hdashline  
\textcolor{blue}{  game }&\textcolor{blue}{   company} &  \textcolor{blue}{ george bush }&\textcolor{blue}{   student }& \textcolor{blue}{  government }\\ 
 (0.4498)&(0.5117) & (0.4728)& (0.5138)&(0.4567)\\ \hdashline  
\textcolor{blue}{  season }& \textcolor{blue}{  million} & \textcolor{blue}{  bush} & \textcolor{blue}{  program }& \textcolor{blue}{  US}\\ 
 (0.4369)&(0.3497) & (0.4509)& (0.1617)&(0.3208)\\ \hdashline  
\textcolor{blue}{  player} &\textcolor{blue}{   companies} & \textcolor{blue}{  campaign  }& \textcolor{blue}{  teacher }& \textcolor{blue}{  united states}\\ 
 (0.3833)&(0.2868) & (0.3285)& (0.1549)&(0.3064)\\ \hdashline 
\textcolor{blue}{  play} &\textcolor{blue}{   stock }& \textcolor{blue}{ president } & \textcolor{blue}{  children} & \textcolor{blue}{ attack }\\ 
 (0.2920)&(0.2746) & (0.3001)& (0.1499)&(0.2796)\\ 
\hline
 \textcolor{red}{  team} &\textcolor{red}{   percent} & \textcolor{red}{  al gore }& \textcolor{red}{  school} &  \textcolor{red}{ official }\\ 
 (0.6119)&(0.6771) & (0.6115)& (0.8144)&(0.7184)\\ \hdashline 
\textcolor{red}{  game }&\textcolor{red}{   company} &  \textcolor{red}{ george bush }&\textcolor{red}{   student }& \textcolor{red}{  government }\\ 
 (0.4498)&(0.5114) & (0.4710)& (0.5138)&(0.4567)\\ \hdashline   
\textcolor{red}{  season }& \textcolor{red}{  million} & \textcolor{red}{  bush} & \textcolor{red}{  program }& \textcolor{red}{  US}\\ 
 (0.4368)&(0.3495) & (0.4540)& (0.1616)&(0.3209)\\ \hdashline   
\textcolor{red}{  player} &\textcolor{red}{   companies} & \textcolor{red}{  campaign  }& \textcolor{red}{  teacher }& \textcolor{red}{  united states}\\ 
 (0.3833)&(0.2867) & (0.3284)& (0.1549)&(0.3065)\\ \hdashline   
\textcolor{red}{  play} &\textcolor{red}{   stock }& \textcolor{red}{ president } & \textcolor{red}{  children} & \textcolor{red}{ attack }\\ 
 (0.2920)&(0.2746) & (0.3003)& (0.1500)&(0.2796)\\ 
\hline
 \textcolor{purple!50!blue}{percent  } &\textcolor{purple!50!blue}{ team } & \textcolor{purple!50!blue}{  al gore }& \textcolor{purple!50!blue}{  school} &  \textcolor{purple!50!blue}{ official }\\ 
 (0.6767)&(0.6119) & (0.6114)& (0.8144)&(0.7183)\\ \hdashline 
\textcolor{purple!50!blue}{company}&\textcolor{purple!50!blue}{game} &  \textcolor{purple!50!blue}{ george bush }&\textcolor{purple!50!blue}{   student }& \textcolor{purple!50!blue}{  government }\\ 
 (0.5118)&(0.4498) & (0.4708)& (0.5139)&(0.4571)\\ \hdashline  
\textcolor{purple!50!blue}{million}& \textcolor{purple!50!blue}{ season } & \textcolor{purple!50!blue}{  bush} & \textcolor{purple!50!blue}{  program }& \textcolor{purple!50!blue}{  US}\\ 
 (0.3497)&(0.4368) & (0.4543)& (0.1615)&(0.3208)\\ \hdashline  
\textcolor{purple!50!blue}{companies} &\textcolor{purple!50!blue}{ player } & \textcolor{purple!50!blue}{  campaign  }& \textcolor{purple!50!blue}{  teacher }& \textcolor{purple!50!blue}{  united states}\\ 
 (0.2868)&(0.3833) & (0.3284)& (0.1549)&(0.3064)\\ \hdashline   
\textcolor{purple!50!blue}{ stock } &\textcolor{purple!50!blue}{play}& \textcolor{purple!50!blue}{ president } & \textcolor{purple!50!blue}{  children} & \textcolor{purple!50!blue}{ attack }\\ 
 (0.2746)&(0.2920) & (0.3003)& (0.1500)&(0.2796)\\
\bottomrule
\end{tabular}
\end{center}
\caption{\add{First 5 sparse PCs for NYTimes data set for $L_2$ variance, with $L_0$ constraint / $L_1$ constraint / $L_0$ penalty / $L_1$ penalty
(the values inside the parenthesis are corresponding to each word in the specified PCs).}}\label{tab:LTC_L2}
\end{table}

\begin{table}[h!]
\begin{center}
\footnotesize 
\begin{tabular}{ c c c c c }\toprule
NYT 1$^{st}$ PC & NYT 2$^{nd}$ PC & NYT 3$^{rd}$ PC & NYT 4$^{th}$ PC & NYT 5$^{th}$ PC \\ \midrule
 \textcolor{green!70!black}{  percent} &\textcolor{green!70!black}{   team} & \textcolor{green!70!black}{ official}& \textcolor{green!70!black}{  school} &  \textcolor{green!70!black}{united states}\\ 
 (0.6047)&(0.5557) & (0.5846)& (0.6433)&(0.4945)\\ \hdashline  
\textcolor{green!70!black}{  company }&\textcolor{green!70!black}{   game} &  \textcolor{green!70!black}{ government }&\textcolor{green!70!black}{   book }& \textcolor{green!70!black}{  country }\\ 
 (0.4915)& (0.4780)&(0.4789) &(0.4421) &(0.4631)\\ \hdashline  
\textcolor{green!70!black}{  million }& \textcolor{green!70!black}{  season} & \textcolor{green!70!black}{  bush} & \textcolor{green!70!black}{  al gore }& \textcolor{green!70!black}{  attack}\\  
 (0.4900)& (0.4499)&(0.4446)& (0.3809)&(0.4353)\\ \hdashline  
\textcolor{green!70!black}{  companies} &\textcolor{green!70!black}{   player} & \textcolor{green!70!black}{ president }& \textcolor{green!70!black}{  student }& \textcolor{green!70!black}{  US}\\  
(0.2926) & (0.3615)& (0.3975)&(0.36535) &(0.4308)\\ \hdashline  
\textcolor{green!70!black}{  market} &\textcolor{green!70!black}{   play }& \textcolor{green!70!black}{ george bush } & \textcolor{green!70!black}{  children} & \textcolor{green!70!black}{ leader }\\
 (0.2585)& (0.3598)&(0.2701)& (0.3348)&(0.4072)\\ 
\hline
  \textcolor{blue}{  percent} &\textcolor{blue}{   team} & \textcolor{blue}{ official}& \textcolor{blue}{  campaign} &  \textcolor{blue}{ school }\\ 
 (0.6047)&(0.5557) & (0.5936)& (0.5413)&(0.6643)\\ \hdashline  
\textcolor{blue}{  company }&\textcolor{blue}{   game} &  \textcolor{blue}{ government }&\textcolor{blue}{  george bush}& \textcolor{blue}{  women }\\ 
 (0.4915)& (0.4780)&(0.4955) &(0.4812) &(0.4544)\\ \hdashline  
\textcolor{blue}{  million }& \textcolor{blue}{  season} & \textcolor{blue}{  bush} & \textcolor{blue}{  al gore }& \textcolor{blue}{ student}\\  
 (0.4900)& (0.4499)&(0.4511)& (0.4702)&(0.3919)\\ \hdashline  
\textcolor{blue}{  companies} &\textcolor{blue}{   player} & \textcolor{blue}{ president }& \textcolor{blue}{election}& \textcolor{blue}{children}\\  
(0.2926) & (0.3615)& (0.3699)&(0.3905) &(0.3578)\\ \hdashline  
\textcolor{blue}{  market} &\textcolor{blue}{   play }& \textcolor{blue}{political} & \textcolor{blue}{palestinian} & \textcolor{blue}{ tax }\\
 (0.2585)& (0.3598)&(0.2481)& (0.3189)&(0.2654)\\ 
\hline
  \textcolor{red}{  percent} &\textcolor{red}{   team} & \textcolor{red}{ official}& \textcolor{red}{  school} &  \textcolor{red}{ billion }\\ 
 (0.6047)&(0.5741) & (0.5487)& (0.5814)&(0.5698)\\ \hdashline  
\textcolor{red}{  company }&\textcolor{red}{   game} &  \textcolor{red}{ government }&\textcolor{red}{group }& \textcolor{red}{  business }\\ 
 (0.4915)& (0.4711)&(0.4936) &(0.5362) &(0.5134)\\ \hdashline  
\textcolor{red}{  million }& \textcolor{red}{  season} & \textcolor{red}{  bush} & \textcolor{red}{  program }& \textcolor{red}{ fund}\\  
 (0.4900)& (0.4432)&(0.4408)& (0.3838)&(0.4105)\\ \hdashline  
\textcolor{red}{  companies} &\textcolor{red}{   player} & \textcolor{red}{ president }& \textcolor{red}{  george bush }& \textcolor{red}{ money}\\  
(0.2926) & (0.3562)& (0.4100)&(0.3473) &(0.4093)\\ \hdashline  
\textcolor{red}{  market} &\textcolor{red}{   play }& \textcolor{red}{group } & \textcolor{red}{student} & \textcolor{red}{stock}\\
 (0.2585)& (0.3534)&(0.2775)& (0.3261)&(0.2747)\\ 
\hline
  \textcolor{purple!50!blue}{  percent} &\textcolor{purple!50!blue}{   team} & \textcolor{purple!50!blue}{ official}& \textcolor{purple!50!blue}{  school} &  \textcolor{purple!50!blue}{ group }\\ 
 (0.6047)&(0.5557) & (0.5856)& (0.6527)&(0.5768)\\ \hdashline  
\textcolor{purple!50!blue}{  company }&\textcolor{purple!50!blue}{   game} &  \textcolor{purple!50!blue}{ government }&\textcolor{purple!50!blue}{   program }& \textcolor{purple!50!blue}{  united states }\\ 
 (0.4915)& (0.4780)&(0.4788) &(0.4523) &(0.4750)\\ \hdashline  
\textcolor{purple!50!blue}{  million }& \textcolor{purple!50!blue}{  season} & \textcolor{purple!50!blue}{  bush} & \textcolor{purple!50!blue}{  student }& \textcolor{purple!50!blue}{ US}\\  
 (0.4900)& (0.4499)&(0.4463)& (0.3628)&(0.3903)\\ \hdashline  
\textcolor{purple!50!blue}{  companies} &\textcolor{purple!50!blue}{   player} & \textcolor{purple!50!blue}{ president }& \textcolor{purple!50!blue}{  family }& \textcolor{purple!50!blue}{  american}\\  
(0.2926) & (0.3615)& (0.3951)&(0.3458) &(0.3861)\\ \hdashline  
\textcolor{purple!50!blue}{  market} &\textcolor{purple!50!blue}{   play }& \textcolor{purple!50!blue}{ al gore  } & \textcolor{purple!50!blue}{  children} & \textcolor{purple!50!blue}{ attack }\\
 (0.2585)& (0.3598)&(0.2690)& (0.3435)&(0.3742)\\ 
\bottomrule
\end{tabular}
\end{center}
\caption{
\add{First 5 sparse PCs for NYTimes data set for $L_1$ variance, with $L_0$ constraint / $L_1$ constraint / $L_0$ penalty / $L_1$ penalty
(the values inside the parenthesis are corresponding to each word in the specified PCs).}}\label{tab:LTC_L1}
\end{table}

\subsection{Implementation details} For single and multi-core architectures we developed our codes using the CBLAS interface. In particular, we use both the GSL BLAS and the Intel MKL\footnote{\href{http://software.intel.com/en-us/articles/intel-mkl/}{http://software.intel.com/en-us/articles/intel-mkl/}} implementations (single-core) and the GotoBLAS2\footnote{\href{https://www.tacc.utexas.edu/research-development/tacc-software/gotoblas2}{https://www.tacc.utexas.edu/research-development/tacc-software/gotoblas2}} and Intel MKL implementations (multi-core). Parallelization in the multi-core case is performed by the OpenMP interface. When comparing the performance of single-core and multi-core architectures, we use Intel MKL library for both serial and parallel versions of the same algorithm for consistency. Nevertheless, in our experience, GotoBLAS2 implementation of these algorithms are faster than the Intel MKL implementation. We use CuBLAS\footnote{\href{http://developer.nvidia.com/cublas}{http://developer.nvidia.com/cublas}}, version 4.0, on GPU (and make use of Thrust whenever possible for operations such as sorting, memory arrangements and data allocation on GPU). For comparisons between single-core and GPU architectures, we use the GSL BLAS implementation on the single-core. On a cluster, linear algebra is done with Intel MKL's PBLAS, while communication between nodes is via MPI.

\section{Conclusion}

We propose a unifying framework for solving 8 SPCA formulations in which all  have the same form and are solved by the same algorithm: the alternating maximization (AM) method. We observed that AM is in all cases equivalent to the GPower method applied to a suitable convex function. Five of these formulations were previously studied in the literature and three were not; notably the $L_1$ constrained $L_1$ (robust) variance seems to be new. For each of these formulations we have written 4 efficient codes---one serial and three parallel---aimed at single-core, multi-core and GPU workstations and a cluster. All these codes are enabled with efficient parallel implementations of a multiple-starting-point globalization strategy which aims to find PCs explaining more variance; with speedup per starting point achieving up to two orders of magnitude. The most efficient of these implementations is ``on-the-fly''. We  demonstrated that our cluster code is able to solve a very large problem with a 357 GB fully dense data matrix.

%\begin{acknowledgements}
%If you'd like to thank anyone, place your cmments here
%and remove the percent signs.
%\end{acknowledgements}

% Authors must disclose all relationships or interests that 
% could have direct or potential influence or impart bias on 
% the work: 
%
% \section*{Conflict of interest}
%
% The authors declare that they have no conflict of interest.

% BibTeX users please use one of
\bibliographystyle{spbasic}      % basic style, author-year citations
\bibliographystyle{spmpsci}      % mathematics and physical sciences
\bibliography{24am-biblio}

\end{document}